\documentclass[letterpaper,11pt]{article}
\usepackage[margin=1in]{geometry}
\usepackage[bookmarks, colorlinks=true, plainpages = false, citecolor = blue,linkcolor=red,urlcolor = blue, filecolor = blue]{hyperref}
\usepackage{url}
\usepackage{amsmath,amsfonts,amsthm,amssymb,bm, verbatim,dsfont,mathtools}
\usepackage{algorithm,algorithmic,color,graphicx,appendix}
\usepackage{amsmath,amsfonts,amscd,amssymb,bm,epsfig,epsf,bbm,constants}
\usepackage{cases}
\usepackage{subfigure}
\usepackage{float}
\usepackage{etoolbox}
\usepackage{tikz}
\usepackage{xr,xspace}
\usepackage{todonotes}
\usepackage{paralist}
\usepackage{caption}
\usepackage{prettyref,microtype}
\usepackage{arydshln}

\newrefformat{alg}{Algorithm~\ref{#1}}
\theoremstyle{plain}
\newtheorem{theorem}{Theorem}
\newtheorem{lemma}{Lemma}

\newtheorem{corollary}{Corollary}
\theoremstyle{definition}
\newtheorem{definition}{Definition}

\newtheorem{remark}{Remark}
\newtheorem{assumption}{Assumption}

\newcommand{\Perp}{\perp}

\newcommand{\argmax}{\mathop{\arg\max}}

\usepackage{xspace,prettyref}

\newcommand{\diverge}{\to\infty}

\newcommand{\iiddistr}{{\stackrel{\text{\iid}}{\sim}}}
\newcommand{\ones}{\mathbf 1}
\newcommand{\zeros}{\mathbf 0}
\newcommand{\reals}{{\mathbb{R}}}

\newcommand{\supp}{{\rm supp}}
\newcommand{\eexp}{{\rm e}}

\newcommand{\identity}{\mathbf I}
\newcommand{\allones}{\mathbf J}

\newcommand{\diff}{{\rm d}}

\newcommand{\Expect}{\mathbb{E}}
\newcommand{\expect}[1]{\mathbb{E}\left[ #1 \right]}
\newcommand{\eexpect}[1]{\mathbb{E}[ #1 ]}

\newcommand{\pprob}[1]{ \mathbb{P}\{ #1 \} }
\newcommand{\prob}[1]{ \mathbb{P}\left\{ #1 \right\} }

\newcommand{\var}{\mathsf{var}}

\newcommand{\Bern}{{\rm Bern}}
\newcommand{\Binom}{{\rm Binom}}

\newcommand{\eg}{e.g.\xspace}
\newcommand{\ie}{i.e.\xspace}
\newcommand{\iid}{i.i.d.\xspace}
\newrefformat{eq}{(\ref{#1})}
\newrefformat{chap}{Chapter~\ref{#1}}
\newrefformat{sec}{Section~\ref{#1}}
\newrefformat{algo}{Algorithm~\ref{#1}}
\newrefformat{fig}{Fig.~\ref{#1}}
\newrefformat{tab}{Table~\ref{#1}}
\newrefformat{rmk}{Remark~\ref{#1}}
\newrefformat{clm}{Claim~\ref{#1}}
\newrefformat{def}{Definition~\ref{#1}}
\newrefformat{cor}{Corollary~\ref{#1}}
\newrefformat{lmm}{Lemma~\ref{#1}}
\newrefformat{prop}{Proposition~\ref{#1}}
\newrefformat{app}{Appendix~\ref{#1}}
\newrefformat{hyp}{Hypothesis~\ref{#1}}
\newrefformat{thm}{Theorem~\ref{#1}}
\newrefformat{ass}{Assumption~\ref{#1}}

\newcommand{\pth}[1]{\left( #1 \right)}
\newcommand{\qth}[1]{\left[ #1 \right]}
\newcommand{\sth}[1]{\left\{ #1 \right\}}

\newcommand{\norm}[1]{\left\|{#1} \right\|}

\newcommand{\iprod}[2]{\left \langle #1, #2 \right\rangle}
\newcommand{\Iprod}[2]{\langle #1, #2 \rangle}
\newcommand{\indc}[1]{{\mathbf{1}_{\left\{{#1}\right\}}}}

\newcommand{\diag}[1]{\mathsf{diag} \left\{ {#1} \right\} }

\newcommand{\tA}{{\widetilde{A}}}

\newcommand{\tZ}{{\widetilde{Z}}}

\newcommand{\calB}{{\mathcal{B}}}

\newcommand{\calE}{{\mathcal{E}}}

\newcommand{\calG}{{\mathcal{G}}}

\newcommand{\calN}{{\mathcal{N}}}

\newcommand{\calS}{{\mathcal{S}}}

\newcommand{\calX}{{\mathcal{X}}}

\newcommand{\bfd}{\mathbf{d}}

\newcommand{\SDP}{{\rm SDP}\xspace}

\newcommand{\ER}{Erd\H{o}s-R\'enyi\xspace}

\newcommand{\Tr}{\mathsf{Tr}}
\renewcommand{\hat}{\widehat}
\renewcommand{\tilde}{\widetilde}

\begin{document}

\title{Semidefinite Programs for Exact Recovery of a Hidden Community}
\date{\today}
\author{ 
Bruce Hajek \and Yihong Wu \and Jiaming Xu\thanks{
B.~Hajek and Y.~Wu are with
the Department of ECE and Coordinated Science Lab, University of Illinois at Urbana-Champaign, Urbana, IL, \texttt{\{b-hajek,yihongwu\}@illinois.edu}.
J.~Xu is with the Simons Institute for the Theory of Computing, University of California, Berkeley, Berkeley, CA, 
\texttt{jiamingxu@berkeley.edu}.}
}

\maketitle

\begin{abstract}
We study a semidefinite programming (SDP) relaxation of the maximum likelihood estimation
for exactly recovering a hidden community of cardinality $K$ from an $n \times n$ symmetric
data matrix $A$, where for distinct indices $i,j$, $A_{ij} \sim P$ if $i, j$ are both in the community and 
$A_{ij} \sim Q$ otherwise, for two known probability distributions $P$ and $Q$.
We identify a  sufficient
condition  and a  necessary condition for the success of SDP  for the general model.
For both the Bernoulli case ($P={{\rm Bern}}(p)$ and $Q={{\rm Bern}}(q)$ with $p>q$) and the Gaussian case
($P=\mathcal{N}(\mu,1)$ and $Q=\mathcal{N}(0,1)$ with $\mu>0$), 
which correspond to the problem of planted dense subgraph recovery and submatrix
localization respectively,
the general results lead to the following
findings: (1) If $K=\omega( n /\log n)$, SDP attains the information-theoretic recovery limits
with sharp constants; (2) If $K=\Theta(n/\log n)$, SDP is order-wise optimal, but strictly
suboptimal by a constant factor; (3) If $K=o(n/\log n)$ and $K \to \infty$, SDP is order-wise suboptimal.
The same critical scaling for $K$ is found to hold, up to constant factors, for the performance of SDP on the stochastic block model of $n$ vertices partitioned into multiple communities of equal size $K$. A key ingredient in the proof of the necessary condition is a construction of a primal feasible solution based on 
 random perturbation of the true cluster matrix.
\end{abstract}

\section{Introduction}

\subsection{Motivation and problem setup}
Consider the stochastic block model (SBM)~\cite{Holland83} with a single community, where out of $n$ vertices a community consisting
of $K$ vertices are chosen uniformly at random; two vertices are connected by an edge with probability $p$
if they both belong to the community and with probability $q$ if either one of them is not in the community. 
The goal is to recover the community based on observation of the graph, which, when $p>q$, 
is also known as the {\em planted dense subgraph} recovery problem~\cite{McSherry01, ChenXu14, HajekWuXu14,Montanari:15OneComm}. 

In the special case of $p=1$ and $q=1/2$, planted dense subgraph recovery reduces to the widely-studied
planted clique problem, \ie, finding a hidden clique of size $K$
in the \ER random graph $\calG(n,1/2)$.  
It is well-known that the maximum likelihood estimator (MLE), which is computationally intractable, finds any clique of size $K \ge 2(1+\epsilon) \log_2 n$ for any constant $\epsilon>0$; however, existing polynomial-time algorithms, including spectral methods  \cite{Alon98}, message passing \cite{Deshpande12}, and semi-definite programming (SDP) relaxation of MLE \cite{FK00}, are only known to find  a clique of size $K \ge \epsilon \sqrt{n}$. 
In fact, impossibility results for the more powerful $s$-round Lov\'asz-Schrijver relaxations and, more recently, degree-$2r$ sums-of-squares (SOS)  relaxation (with $s=1$ and $r=1$ corresponding to SDP) have been recently obtained in \cite{FK03} and \cite{DeshpandeMontanari15,Meka15,HKP15,RS15}, showing that 
relaxations of constant rounds or degrees lead to order-wise suboptimality even for detecting the clique.
In other words, for the planted clique problem there is a significant gap between the state of the art of polynomial-time  algorithms 
and what is information-theoretically possible. 

In sharp contrast, for sparser graphs and larger community size, SDP relaxations have been shown to achieve the information-theoretic recovery limit up to sharp constants. For 
$p=a \log n/n, q=b \log n/n$ and $K=\rho n$ for fixed constants $a,b>0$ and $0< \rho<1$, the
recent work \cite{HajekWuXuSDP14} identified a sharp threshold $\rho^\ast=\rho^\ast(a,b)$ such that if $ \rho > \rho^\ast$, an SDP relaxation of MLE recovers the hidden community with high probability; if $\rho <\rho^\ast$, exact recovery is information theoretically impossible.
This optimality result of SDP has been extended to 
multiple communities as long as their sizes scale linearly with the graph size $n$ \cite{Bandeira15,HajekWuXuSDP15,ABBK,perry2015semidefinite,MontanariSen15}.

The dichotomy between the optimality of SDP up to sharp constants in the relatively sparse regime and  the order-wise suboptimality 
of SDP in the dense regime 
prompts us to investigate the following question:
\begin{quote}
\emph{When do SDP relaxations cease to be optimal for planted dense subgraph recovery}?
\end{quote}

In this paper, we address this question under the more general hidden community model considered in  \cite{Deshpande12}. 

\begin{definition}[Hidden Community Model]  \label{def:model}
 Let $C^*$ be drawn uniformly at random from all subsets of $[n]$ of cardinality $K$.
 Given probability measures $P$ and $Q$ on a common measurable space $\calX$, let $A$ be an $n \times n$ symmetric matrix with zero diagonal
 where for all $1 \le i<j \le n$, $A_{ij}$ are mutually independent, and $A_{ij} \sim P$ if $i,j\in C^*$ and $A_{ij} \sim  Q$ otherwise.
\end{definition}
The distributions $P$ and $Q$ as well as the community size $K$ vary with the matrix size $n$ in general.
In this paper we assume that these model parameters
are known to the  estimator, and focus on exact recovery of the hidden community based on the data matrix $A$, namely, constructing an estimator 
$\hat{C}=\hat{C}(A)$, such that as $n \to \infty$,
$
\pprob{\hat C\neq C^*} \to 0
$
uniformly in the choice of the true cluster $C^*$.

We are particularly interested in the following choices of $P$ and $Q$:
\begin{itemize}
\item {Bernoulli case}:  $P=\Bern(p)$ and $Q=\Bern(q)$ with $0 \le q < p \le 1$. 
In this case, the data matrix $A$ corresponds to the adjacency matrix of a graph, and
the problem reduces to planted dense subgraph
recovery. 

\item {Gaussian case}: $P=\calN(\mu,1)$ and $Q=\calN(0,1)$ with $\mu>0$. 
In this case, the submatrix of $A$ with row and column indices in $C^\ast$ has a positive mean $\mu$ except on the diagonal, while the rest of $A$ has zero mean, and the problem corresponds to a symmetric version of the {\em submatrix localization} problem studied in \cite{shabalin2009submatrix,kolar2011submatrix,Butucea2013sharp,ma2013submatrix,ChenXu14,CLR15}. 
\end{itemize}

\subsection{Main results}
\label{sec:main}
We show that for both planted dense subgraph recovery and submatrix localization, SDP relaxations of MLE achieve the information-theoretic optimal threshold if and only if  the hidden community size satisfies $K=\omega(\frac{n}{\log n})$.   More specifically, 
 \begin{itemize}
	\item $K=\omega(\frac{n}{\log n})$, SDP attains the information-theoretic recovery limits
with sharp constants. This extends the previous result in \cite{HajekWuXuSDP14} obtained for $K=\Theta(n)$ and the Bernoulli case. 
\item $K=\Theta(\frac{n}{\log n})$, SDP is order-wise optimal, but strictly
suboptimal by a constant factor; 
\item $K=o(\frac{n}{\log n})$ and $K \to \infty$, SDP is order-wise suboptimal.
\end{itemize}

To establish our main results, we derive a sufficient condition and a necessary condition under which 
 the optimal solution to SDP is unique and coincides with the true cluster matrix. 
 In particular, for planted dense subgraph recovery, whenever SDP does not achieve the information-theoretic threshold, our sufficient condition and necessary condition are within constant factors of each other; for submatrix localization, we characterize the minimal signal-to-noise ratio required by SDP within a factor of four when $K=\omega(\sqrt{n})$.
The sufficiency proof is similar to those in  \cite{HajekWuXuSDP14} based on the dual certificate argument; we extend the construction and validation of dual certificates for the success of SDP to the general distributions $P, Q$. 
The necessity proof is via constructing a high-probability feasible solution to the SDP by means of random perturbation of the ground truth that leads to a higher objective value. One could instead adapt the existing constructions in the SOS literature for planted clique~\cite{DeshpandeMontanari15,Meka15,HKP15,RS15} to our setting, but it falls short of establishing the impossibility of SDP to attain the optimal recovery threshold in the critical regime of $K=\Theta(n/\log n)$; see \prettyref{rmk:primalpf} for details.

An alternative approach to establish impossibility results for SDP, thanks to strong duality that holds for the specific program, is to prove the non-existence of dual certificates, which turns out to yield the same condition given by the aforementioned explicit construction of primal solutions.
The dual-based method has been previously used for proving necessary conditions for related nuclear-norm constrained optimization problems,
see e.g., \cite{kolar2011submatrix,vinayak2013sharp,ChenXu14}; however, the constants in the derived conditions  
are often loose or unspecified.
 In comparison, we aim to obtain necessary conditions for
 SDP relaxations with explicit constants. Another difference 
 is that the specific SDP considered here is more complicated involving the stringent positive semi-definite constraint and a set of equality and non-negativity constraints. 

Using similar techniques, we obtain analogous results for SDP relaxation for SBM with logarithmically many communities.
Specifically, consider the network of   $n=rK$  vertices partitioned into
$r$ communities of  cardinality $K$ each, with edge probability $p$ for pairs of vertices within communities
and $q$ for other pairs of vertices. Then SDP relaxation, in contrast to the MLE, 
is constantwise suboptimal if $r \geq C \log n$ for sufficiently large $C$,
and orderwise suboptimal if $r =\omega \left(\log n \right)$.   That is, it is
constantwise suboptimal if $K \leq \frac {c n}{\log n}$ for sufficiently small $c$,
and orderwise suboptimal if $K =o  \left( \frac{ n}{ \log n} \right)$.  
This result complements the sharp optimality for SDP previously established in \cite{HajekWuXuSDP15} for $r=O(1)$
and extended to $r=o(\log n)$ in \cite{ABBK}.

 In closing, we comment on the barrier which prevents SDP from being optimal. 
 It is known that, see \eg, \cite{ChenXu14,MontanariSen15},  spectral methods which estimate the communities based on the 
leading eigenvector of the data matrix $A$ suffer from a spectral barrier: 
the spectrum of the ``signal part'' $\expect{A}$ must escape that of the ``noise part'' $A-\expect{A}$, 
\ie, the smallest nonzero singular value of $\expect{A}$ needs to be much larger than the spectral norm $\|A-\expect{A}\|$. 
Closely related to the spectral barrier,  the SDP barrier  
originates from a  key random quantity (see \prettyref{eq:Vm}),
which is at most and, in fact, possibly much smaller than, the largest eigenvalue of $A-\expect{A}$. 
Thus we expect the SDP barrier to be weaker than the spectral one. Indeed, for the submatrix localization problem, if the submatrix size is sufficiently small, \ie, $K=o(\sqrt{n/\log n})$, SDP recovers the community with high probability if $\mu=\Omega(\sqrt{\log n})$, while
the spectral barrier requires a much stronger signal: $\mu=\Omega( \sqrt{n}/K)$; see \prettyref{sec:gauss-info} for details.

\subsection{Notation}
Let $\identity $ and $\allones$ denote the identity matrix and all-one matrix, respectively.
For a matrix $X$ we write  $X \succeq 0$ if $X$ is positive semidefinite, and $X \ge 0$ if $X$ is non-negative entrywise.
Let $\calS^n$ denote the set of all $n \times n$ symmetric matrices. For $X \in \calS^n$, let $\lambda_2(X)$ denote its second smallest eigenvalue.
For an $m\times n$ matrix $M$, let $\|M\|$ and $\|M\|_{\rm F}$  denote its spectral and Frobenius norm, respectively.
For any $S\subset [m], T \subset [n]$, let $M_{ST}\in \reals^{S \times T}$ denote $(M_{ij})_{i\in S,j \in T}$ and for $m=n$ abbreviate $M_{S}=M_{SS}$.
For a vector $x$, let $\|x\|$ denote its Euclidean norm.
We use standard big $O$ notations as well as their counterparts in probability,
e.g., for any sequences $\{a_n\}$ and $\{b_n\}$, $a_n=\Theta(b_n)$ or $a_n  \asymp b_n$
if there is an absolute constant $c>0$ such that $1/c\le a_n/ b_n \le c$.
All logarithms are natural and we adopt the convention $0 \log 0=0$. 
Let $\Bern(p)$ denote the Bernoulli distribution with mean $p$ and
$\Binom(n,p)$ denote the binomial distribution with $n$ trials and success probability $p$.
Let $d(p\|q) = p \log \frac{p}{q} + (1-p)\log \frac{1-p}{1-q}$ denote the Kullback-Leibler (KL) divergence between
$\Bern(p)$ and $\Bern(q)$.
We say a sequence of events $\calE_n$ holds with high probability, if $\prob{\calE_n} \to 1$ as $n \to \infty$.

\section{Semidefinite programming relaxations}
Recall that $\xi^\ast \in \{0,1\}^n$  denotes the indicator of 
the underlying cluster $C^\ast$, 
such that $\xi^\ast_i=1$ if $ i \in C^*$ and $\xi^\ast_i=0$ otherwise.
Let $L$ denote an $n \times n$ symmetric matrix such that 
$L_{ij}=f(A_{ij})$ for $i \neq j$ and $L_{ii}=0$, where $f: \calX \to \reals$ is any 
function possibly depending on the model parameters. 
Consider the following combinatorial optimization problem:
\begin{align}
\hat{\xi} = \argmax_{\xi}  & \; \sum_{i,j} L_{ij} \xi_i \xi_j    \nonumber  \\
\text{s.t.	}  & \; \xi \in \{ 0, 1 \}^n    \label{eq:PDSML1_SL}    \\
 & \; \xi^\top \mathbf{1} =K,  \nonumber
\end{align}
which maximizes the sum of entries among all $K \times K$ principal submatrices of $L$.

If $L$ is the log likelihood ratio (LLR) matrix with $f( A_{ij} )=\log \frac{dP}{dQ}(A_{ij} )$ for $i \neq j$ and $L_{ii}=0$, 
then $\hat{\xi} $ is precisely the MLE of $\xi^\ast$. 
In general, evaluating the MLE requires knowledge of $K$ and the distributions $P, Q$. 
Computing the MLE is NP hard in the worst case for general values of $n$ and $K$ since certifying the existence of a clique of a specified size in an
undirected graph, which is known to be NP complete \cite{Karp72},
can be reduced to computation of the MLE.  
This intractability of the MLE prompts us to consider its semidefinite programming relaxation as studied in \cite{HajekWuXuSDP14}.
Note that \prettyref{eq:PDSML1_SL} can be equivalently\footnote{Here \prettyref{eq:PDSML1_SL} and \prettyref{eq:PDSML2_SL} are equivalent in the following sense: For any feasible $\xi$ for \prettyref{eq:PDSML1_SL},
$Z=\xi\xi^\top$ is feasible for \prettyref{eq:PDSML2_SL}; Any feasible $Z$ for \prettyref{eq:PDSML2_SL} can be written as $Z=\xi\xi^\top$ such that either $\xi$ or $-\xi$
is feasible for \prettyref{eq:PDSML1_SL}. } formulated  as
\begin{align}
\max_{Z}  & \; \Iprod{L}{Z} \nonumber  \\
\text{s.t.	} & \; \text{rank}(Z)=1, \nonumber  \\
& \;  Z_{ii} \le 1, \quad \forall i \in [n]\nonumber \\
 & \; Z_{ij} \ge 0,  \quad \forall i, j \in [n]  \nonumber \\
 & \;  \Iprod{\identity}{Z} = K \nonumber  \\
 & \; \Iprod{\allones}{Z} = K^2.
\label{eq:PDSML2_SL}
\end{align}

Replacing the rank-one constraint by the positive semidefinite constraint leads to the following convex relaxation of \prettyref{eq:PDSML2_SL}, which can be cast as a semidefinite program:\footnote{$\widehat{Z}_{\SDP}$ as $\argmax$ denotes the set of maximizers of the optimization problem \prettyref{eq:PDSCVX_SL}. If $Z^*$ is the unique maximizer, we write $\widehat{Z}_{\SDP} = Z^*$.}
\begin{align}
\widehat{Z}_{\SDP} = \argmax_{Z}  &\; \langle L, Z \rangle     \nonumber     \\
\text{s.t.	} & \; Z \succeq 0   \label{eq:PDSCVX_SL}  \\
& \;  Z_{ii} \le 1, \quad \forall i \in [n]\nonumber \\
& \; Z  \ge 0  \nonumber \\
& \;  \Iprod{\identity}{Z} = K  \nonumber  \\
 & \; \Iprod{\allones}{Z} = K^2.  \nonumber
\end{align}
Let $\xi^\ast \in  \{0,1\}^n$ denote the indicator of the community such that $\supp(\xi^\ast)=C^*$.
Let $Z^*=\xi^*(\xi^*)^\top$ denote the cluster matrix corresponding to $C^\ast$.
It is straightforward to retrieve the underlying cluster $C^\ast$ from $Z^*$. 
Thus, if $\pprob{ \widehat{Z}_{\SDP} = Z^* } \to 1$ as $n \to \infty$,
then exact recovery of $C^*$ is attained. Note that by the symmetry of the SDP formulation and the distribution of $L$, the probability of success $\pprob{\widehat{Z}_{\SDP}=Z^\ast}$ is the same conditioned on any realization of $\xi^\ast$ and hence the worst-case probability of error coincides with the average-case one.

Recall that if $L$ is the LLR matrix, then the solution $\hat{\xi} $
to  \prettyref{eq:PDSML1_SL}  is precisely the MLE of $\xi^\ast$. 
In the Gaussian case, $\log \frac{dP}{dQ}(A_{ij} )= \mu ( A_{ij} - \mu /2)$ with $\mu>0$ for $i \neq j$; in the Bernoulli case, $\log \frac{dP}{dQ}(A_{ij} ) = \log \frac{p(1-q)}{q(1-p)} A_{ij} + \log \frac{1-p}{1-q}$ with $p>q$ for $i \neq j$.
Thus, in both cases,   \prettyref{eq:PDSCVX_SL}  with $L=A$ corresponds to a semidefinite programming relaxation of the MLE, and the only model parameter needed for evaluating \prettyref{eq:PDSCVX_SL} is
 the cluster size $K$.

\section{Analysis of SDP in the general model}
In this section,  we give a sufficient condition and a necessary
condition, both deterministic,  for the success of SDP \prettyref{eq:PDSCVX_SL} for exact recovery.   
Define
\begin{equation}
e(i,C^*)=\sum_{j\in C^*} L_{ij}, \quad i \in [n]
	\label{eq:edge}
\end{equation}
and 
\begin{align*}
\alpha &=\mathbb{E}_P[L_{12} ], \\
\beta & = \mathbb{E}_Q[L_{12} ].
\end{align*}
We assume that $ \alpha \ge \beta$, \ie, $L$ has an elevated mean in the submatrix supported on $C^\ast \times C^\ast $ (excluding the 
diagonal).  This assumption guarantees that  $Z^\ast$ is the optimal solution to \prettyref{eq:PDSCVX_SL} when $L$ is replaced by its mean $\expect{L}$,
and is clearly satisfied when $L$ is the LLR matrix, in which case $\alpha=D(P\|Q) \geq 0  \geq -D(Q\|P) = \beta$,  or $L=A$ in the Gaussian and Bernoulli cases.

\begin{theorem}[Sufficient condition for SDP: general case]\label{thm:PlantedGeneralSharp}
If 
\begin{align}
    \min_{i \in C^\ast} e(i, C^*)  -  \max \left\{  \max_{i \notin C^\ast} e(i, C^*), K \beta \right \}   > \| L-\expect{L} \|-  \beta,  \label{eq:SDP_suff1} 
    \end{align}
then $\widehat{Z}_{\SDP}=Z^\ast$.
\end{theorem}

The sufficient condition of \prettyref{thm:PlantedGeneralSharp} is derived via the dual certificate argument.
That is, we give an explicit construction of dual variables which together with
$Z^*$ are shown to satisfy  the  KKT conditions under the condition \prettyref{eq:SDP_suff1}.

The necessary condition relies on the following key quantity which is the value of an auxiliary SDP program.
Let $m=n-K$ and $M=L_{(C^\ast)^c \times (C^\ast)^c}$ denote the submatrix of $L$ outside the community.
Then $M$ is an $m \times m$ symmetric matrix with zero diagonal, where 
$\{M_{ij}:1 \le i<j \le m\}$ are i.i.d.
For $a\in \reals$, consider the value (random variable) of the following SDP:
\begin{align}
V_m(a) \triangleq  \max_{Z}  &\; \Iprod{ M }{ Z}      \label{eq:Vm} \\
\text{s.t.	} & \; Z \succeq 0   \nonumber  \\
 & \; Z \ge 0  \nonumber \\
& \;  \Tr(Z) = 1  \nonumber  \\
 & \; \Iprod{\allones}{Z} = a.  \nonumber
\end{align}
There is no feasible solution to \prettyref{eq:Vm} unless $1 \leq a \leq m$, so by convention, let
$V_m(a)=-\infty$ if $a<1$ or $a>m$. Dropping the second and the last constraints in \prettyref{eq:Vm}, yields
$V_m(a) \leq \lambda_{\max}(M)$. 
Also, $V_m(1)=0,$   $V_m(m)=\Iprod{ M }{ \allones}/m,$  and $a \mapsto V_m(a)$ is concave on $[1,m]$.
Clearly, the distributions of $M$ as well as $V_m(a)$ depend on the distribution $Q$ but not $P$.

Fix $K, n, C^*,$ the matrix $L$, and $a\in [1,K]$.     Also, let $r=\frac{a}{K}$.
For ease of  notation suppose the indices are permuted
so that $C^* = [K],$  index $K$ minimizes
$e(i,C^*)$ over all $i\in C^*,$  and index $K+1$
maximizes $e(j,C^*)$ over all  $j\not\in C^*$.
Let $U$ be an $n\times n$ matrix corresponding to the solution of the SDP defining
$V_{m}(a)$ with $M=L_{(C^*)^c\times (C^*)^c}$  in \prettyref{eq:Vm}.
That is, $U$ is a symmetric $n\times n$ matrix  with $U_{ij}=0$ if $(i,j) \not\in (C^*)^c\times (C^*)^c,$
$V_{m}(a)=\Iprod{L}{U}$,   $U\succeq0$, $U\geq 0$, $\Tr(U)=1,$  and $\Iprod{\allones}{U}=a=Kr$.  

Next we give intuition about the construction of primal feasible solutions via random perturbation that lead to a necessary condition for SDP.
Three positive semidefinite perturbations of $Z^*$, namely $Z^* + \delta_i$ for $1\leq i \leq 3,$ can be defined
for $0 < \epsilon < 1/2$ by letting (dashed lines delineate the $K\times K$ submatrices and only nonzero
entries are shown):
\begin{align}
\delta_1   =&   \left[
    \begin{array}{cccc;{2pt/2pt}c}
            &&    & -\epsilon  &   \hspace{0.8in} \\
         &&      & \vdots &  \\
         &&      &  -\epsilon  &  \\
-\epsilon & \cdots &        -\epsilon        & -2\epsilon + \epsilon^2   \\ \hdashline[2pt/2pt]
&&&\\
&&&\\
&&&\\
    \end{array}
\right]  
\label{eq:delta1}
\\  
\delta_2   =  (1-r) & \left[
    \begin{array}{cccc;{2pt/2pt}cc}
            &&  &  & \epsilon  &   \hspace{0.7in} \\
         &&    &  & \vdots &  \\
         &&    &  &  \epsilon  &  \\  \hdashline[2pt/2pt]
\epsilon & \hspace{.52in} \cdots \hspace{.3in}  &        \epsilon       & &  \epsilon^2   \\ 
&&&\\
&&&\\
&&&\\
    \end{array}
\right] \label{eq:delta2}
\\  
\delta_3  =   2\epsilon & U \label{eq:delta3}
\end{align}

It turns out that $Z^* + \delta_1 + \delta_2+ \delta_3$ is close to being positive semidefinite for small $\epsilon$.   Also,
\begin{align*}
\begin{array}{lll}
\Iprod{\identity}{\delta_1}=-2\epsilon + o(\epsilon)   &   \Iprod{\allones }{\delta_1}=-2K\epsilon + o(\epsilon)    
    &  \Iprod{L}{\delta_1}= -2\epsilon \min_{i\in C^*} e(i,C^*)  \\
\Iprod{\identity}{\delta_2}=   o(\epsilon)   &   \Iprod{\allones }{\delta_2}= 2(K-a) \epsilon + o(\epsilon)  & 
     \Iprod{L}{\delta_2} = 2(1-r) \epsilon \max_{j\notin C^*} e(j,C^*)  \\
\Iprod{\identity}{\delta_3}=  2\epsilon   &   \Iprod{\allones }{\delta_3}= 2a\epsilon   &  \Iprod{L}{\delta_3} = 2\epsilon V_{n-K}(a)   
\end{array}
\end{align*}
Therefore, up to $o(\epsilon)$ terms, $Z^* + \delta_1 + \delta_2+ \delta_3$ satisfies
the two equality constraints of the SDP \eqref{eq:PDSCVX_SL} and is near a feasible solution of the SDP \eqref{eq:PDSCVX_SL}, suggesting that a necessary condition for the optimality of
$Z^*$ is  $\Iprod{L}{\delta_1+\delta_2+\delta_3}\leq 0$.   Note that
$$
\Iprod{L}{ \delta_1+\delta_2+\delta_3} 
= 2\epsilon \left( (1-r)  \max_{j\not\in C^*} e(j,C^*) + V_{n-K}(a)  - \min_{i \in C^*} e(i,C^*)  \right) + o(\epsilon).
$$
Hence the term inside the parenthesis must be non-positive.
This leads the following deterministic necessary condition for SDP.
The proof, given in  \prettyref{sec:proof_nec}, is a minor variation of the heuristic argument just presented.

\begin{theorem}[Necessary condition for SDP: general case]\label{thm:PlantedGeneralSharp_nec}
If $Z^* \in \widehat{Z}_{\SDP} $, then
\begin{equation}   \label{eq:general_nec_cond}
\min_{i \in C^\ast} e(i, C^*) - 	 \max_{j \notin C^\ast} e(j, C^*)    \ge  \sup_{1 \le a \le K } \left\{   V_{n-K}(a)   -  \frac{a}{K} \max_{j \notin C^\ast} e(j, C^*)   \right\}.
\end{equation}
\end{theorem}
\begin{remark}
Note that \eqref{eq:general_nec_cond} is equivalent to
$$
\min_{i \in C^\ast} e(i, C^*)    \ge  \sup_{1 \le a \le K } \left\{ V_{n-K}(a)  +  \left(1 - \frac{a}{K}\right)
\max_{i \notin C^\ast} e(i, C^*)   \right\},
$$
Setting $a=K$ in \eqref{eq:general_nec_cond}
yields the weaker necessary condition:
$\min_{i \in C^\ast} e(i, C^*)  \ge  V_{n-K}(K)$.
\end{remark}

\begin{remark} \label{rmk:primalpf}
The problem formulation as well as proof technique of \prettyref{thm:PlantedGeneralSharp_nec} differ from existing results on the planted clique problem for the sum of squares (SoS) hierarchy \cite{DeshpandeMontanari15,Meka15,HKP15,RS15} in an essential way. Aside from the
fact that those papers consider more powerful convex relaxations, they address the clique detection problem (which do have implications for clique estimation), 
which can be viewed as testing the null hypothesis  $H_0:$ clique absent versus the alternative $H_1:$ clique present, using the value of the SOS program
as the test statistic.   
The approach of these papers involves only the null hypothesis,
showing that 
a feasible solution to SOS program can be constructed based on the $\calG(n,1/2)$ graph 
whose objective value is much larger than the size of the largest clique in $\calG(n,1/2)$, leading to 
a large integrality gap. This further induces a high false-positive error probability 
if the size of the planted clique $K$ is small. 
In comparison, since we are dealing with recovery as opposed to detection using SDP, the impossibility result in \prettyref{thm:PlantedGeneralSharp_nec} follows from the fact that, if the true cluster matrix $Z^*$ is an optimal solution, 
then certain random perturbations of $Z^*$ must not  lead to a strictly larger objective value. 
More precisely, the perturbation argument involves three directions \prettyref{eq:delta1}-\prettyref{eq:delta3}. Note that 
the matrix $U$ in \prettyref{eq:delta3} is the maximizer of \prettyref{eq:Vm} and can be constructed using similar techniques in the SoS literature.
However, this perturbation alone is not enough to separate the performance of SDP from MLE 
in the critical regime $K=\Theta(n/\log n)$, and it is necessary to exploit the other perturbations terms \prettyref{eq:delta1}-\prettyref{eq:delta2} that depend on the true cluster matrix.
\end{remark}

\begin{remark}
Since Slater's condition and hence strong duality holds for the SDP \prettyref{eq:PDSCVX_SL}, the fulfillment of the KKT conditions is  
necessary for $Z^*$ to be a maximizer. We provide an alternative proof of \prettyref{thm:PlantedGeneralSharp_nec}
in \prettyref{sec:proof_nec},  showing that \prettyref{eq:general_nec_cond} is necessary for the existence of dual variables to 
satisfy the KKT conditions together with $Z^*$. 
\end{remark}

By comparing \prettyref{thm:PlantedGeneralSharp} and  \prettyref{thm:PlantedGeneralSharp_nec}, 
we find that both the sufficient and necessary conditions are in terms of the 
separation between $\min_{ i\in C^*} e(i, C^*)$ and $\max_{ j \notin C^*} e(j, C^*)$. 
In comparison, for the  optimal estimator, MLE,  to succeed in exact recovery, it is necessary that 
$\min_{i \in C^*} e(i, C^*) \ge \max_{j \notin C^*} e(j, C^*)$; otherwise, one can
form a candidate community $C$ by swapping the node $i$ in $C^*$ achieving the minimum $e(i, C^*)$ with the node $j$
not in $C^*$ achieving the maximum $e(j, C^*)$, so that 
the new community $C$ has a likelihood at least as large as that of $C^*$. 
 
Capitalizing on Theorems \ref{thm:PlantedGeneralSharp} and \ref{thm:PlantedGeneralSharp_nec}, we will derive explicit sufficient  and necessary  results for 
the success of SDP in the Gaussian and Bernoulli cases. Interestingly, in both cases, if  $K=\omega(n/ \log n)$,
the sufficient condition of SDP coincides in the leading terms with the information-theoretic necessary condition for $\min_{i \in C^*} e(i, C^*) \ge \max_{j \notin C^*} e(j, C^*)$,
thus resulting in the optimality of SDP with the sharp constants.

\section{Submatrix localization}


In this section we consider the submatrix localization problem corresponding to the Gaussian case of \prettyref{def:model}. 
The SDP relaxation of MLE is given by \prettyref{eq:PDSCVX_SL} with $L=A$.

\begin{theorem}[Sufficient conditions for SDP: Gaussian case] \label{thm:PlantedSharp_SL}
Assume that 
$K \ge 2$ and  $n-K \asymp n$. 
Let $\epsilon>0$ be an arbitrary constant. 
If either $K \to \infty$ and 
\begin{align}
 \mu  (1-\epsilon)
\ge &~ \frac{1}{\sqrt{K}} \left(  \sqrt{2 \log K} + \sqrt{2 \log (n-K) }  \right) + \frac{2 \sqrt{n} }{K},  \label{eq:sdp-suff}
 \end{align}
 or
 \begin{equation}
\mu (1-\epsilon) \ge 2 \sqrt{ \log K} + 2 \sqrt{ \log n }, 
	\label{eq:SDP-trivial}
\end{equation}
then  $\pprob{\widehat{Z}_{\SDP}=Z^\ast} \to 1$ as $n \to \infty$.
\end{theorem}

\begin{remark}
To deduce \prettyref{eq:sdp-suff} from the general sufficient condition \prettyref{eq:SDP_suff1}, we first show  that
 \begin{align}
 \min_{i \in C^\ast} e(i, C^*)  &  \ge (K-1) \mu -  \sqrt{2(K-1) \log K} +o_P(\sqrt{K})  \label{eq:mine} \\
 \max_{j \notin C^\ast} e(j, C^*) & \le \sqrt{2 K \log (n-K)} +o_P(\sqrt{K}). \label{eq:maxe}
 \end{align}
Then \prettyref{eq:sdp-suff} follows since $A-\expect{A}$ is an $n\times n$ Wigner matrix whose spectral norm is $(2+o_P(1)) \sqrt{n}$. 
	Under the condition \prettyref{eq:SDP-trivial}, 
	with high probability, $\min_{(i,j) \in C^\ast \times C^\ast: i \neq j} A_{ij} >\max_{(i,j) \notin C^\ast \times C^\ast} A_{ij} $, and thus 
	$\widehat{Z}_{\SDP}=Z^\ast$. In this case, the community can be also trivially recovered with probability tending to one using entrywise hard thresholding, and not surprisingly, but SDP as well.
\end{remark}

Next we present a converse result for the exact recovery performance of SDP in a strong sense:
\begin{theorem}[Necessary condition  for SDP in Gaussian case]
Assume that $L=A$ in the SDP \prettyref{eq:PDSCVX_SL}, $K \to \infty$, and $K=o(n)$.  
Suppose that $\liminf_{n \to \infty}  \pprob{Z^\ast \in \widehat{Z}_{\SDP}} > 0$.
Then for any fixed $\epsilon>0$,
\begin{itemize}
	\item if $K = \omega(\sqrt{n})$, then
\begin{align}
\mu (1+ \epsilon) \ge \frac{1}{\sqrt{K}} \left( \sqrt{2 \log K}  +  \sqrt{2  \log (n-K) }  \right) + \frac{\sqrt{n} }{2K} . 
\label{eq:sdp-nece1}
\end{align}
\item if $ K = \Theta(\sqrt{n})$, then
\begin{equation}
\mu (1+\epsilon) \geq \sqrt{ \log \pth{1+ \frac{n}{4 K^2} } } . 
	\label{eq:sdp-nece2}
\end{equation}
\item if $ K = o(\sqrt{n})$, then
\begin{equation}
\mu (1+\epsilon)  \geq 	\sqrt{ \frac{1}{3} \log \frac{ n}{K^2} }. 
	\label{eq:sdp-nece3}
\end{equation}

\end{itemize}
	\label{thm:sdp-nece}
\end{theorem}

\begin{remark}
To deduce \prettyref{thm:sdp-nece} from the general necessary condition given in 
\prettyref{thm:PlantedGeneralSharp_nec}, we first show that the inequalities in \prettyref{eq:mine}
and \prettyref{eq:maxe} are in fact equalities. 
Then, we prove a high-probability lower bound to $V_{m}(a)$
 and choose $a=o(K)$
in \prettyref{eq:general_nec_cond} when $K=\omega(\sqrt{n})$, 
and $a=K$ when $K=O(\sqrt{n})$. 

By comparing sufficient condition \prettyref{eq:sdp-suff} and necessary condition \prettyref{eq:sdp-nece1},
we can see that the sufficient condition and necessary condition are within a factor of $4$ in the case of $K=\omega(\sqrt{n}).$
\end{remark}

\subsection{Comparison to the information-theoretic limits}
\label{sec:gauss-info}
It is instructive to compare the performance of the SDP to the information-theoretic fundamental limits.
We focus on the most interesting regime of $K \to \infty$ and $n-K \asymp n $. 
It has been shown (cf.~\cite[Theorem 4]{HajekWuXu_one_info_lim15}) that, for any $\epsilon > 0$, the MLE  (which minimizes the probability of error) achieves exact recovery if
\begin{equation}
\mu(1-\epsilon) \geq \frac{1}{\sqrt{K}} \max\sth{\sqrt{2 \log K}+\sqrt{2\log n}, 2 \sqrt{\log (n /K)}};  \label{eq:infor_exact}
\end{equation}
conversely, if
\begin{equation}
\mu(1+\epsilon) \leq \frac{1}{\sqrt{K}} \max\sth{\sqrt{2 \log K}+\sqrt{2\log n}, 2 \sqrt{\log (n /K)}},
  \label{eq:infor_exact_necc}
\end{equation}
no estimator can exactly recover the community with high probability.

%

Comparing \prettyref{eq:infor_exact} -- \prettyref{eq:infor_exact_necc} with \prettyref{eq:sdp-suff}, \prettyref{eq:SDP-trivial},
and  \prettyref{eq:sdp-nece1}-- \prettyref{eq:sdp-nece3}, we arrive at the following conclusion on the performance of the SDP relaxation:
\begin{itemize}

	\item $K=\omega(n/\log n)$: Since $\sqrt{n} = o(\sqrt{ K \log n })$, in this regime SDP attains the information-theoretically optimal recovery threshold with sharp constant.

\item $K=\Theta(n/\log n)$: SDP is order-wise optimal but strictly suboptimal in terms of constants. More precisely, consider the critical regime of
\begin{equation}
	K= \frac{ \rho n }{\log n}, \quad \mu =\frac{\mu_0 \log n}{\sqrt{n}}
	\label{eq:regime}
\end{equation}
 for fixed constants $\rho,\mu_0>0$.
 Then MLE succeeds (resp.~fails) if $\rho \mu_0^2$ $>$ (resp.~$<$) $8$.
If $\rho \mu_0 > 2 \sqrt{2 \rho} + 2$, then SDP succeeds; conversely, if SDP succeeds, then $\rho \mu_0 \ge 2 \sqrt{2 \rho} + 1/2$. Moreover, it is shown in \cite{HajekWuXu_MP_submat15}
that a message passing algorithm plus clean-up
succeeds if $\rho \mu_0^2>8$ and $ \rho \mu_0 >1/\sqrt{\eexp}$, while a linear message passing
algorithm corresponds to a spectral method succeeds if $\rho \mu_0^2 >8$ and $\rho \mu_0 >1$.
Therefore, SDP is strictly suboptimal comparing to MLE,
 message passing, and linear message passing  for $\rho> 0$, $\rho > (1/\sqrt{\eexp} -1/2)^2/8$,
 and $\rho> 1/32$, respectively.
  See \prettyref{fig:phase-sdp} for an illustration.

\item $\omega(1) \leq K=o(n/\log n)$: Comparing to MLE, SDP is order-wise suboptimal. Moreover,
when $K \le n^{1/2 -\delta}$ for any  fixed constant $\delta>0$,
$\mu =\Omega(\sqrt{\log n})$ is necessary for
SDP to achieve exact recovery, while the entrywise hard-thresholding or simply picking the largest entries attains
exact recovery when $\mu (1-\epsilon) \geq 2 \sqrt{\log n}$. Thus in this regime, the more sophisticated SDP procedure
is only possible to outperform the trivial thresholding algorithm by a constant factor. Similar phenomena has been
observed in the bi-clustering problem \cite{kolar2011submatrix}, which is an asymmetric version of the submatrix localization
problem,  and the sparse PCA \cite{krauthgamer2015}.

\item
$K=\Theta(1)$: 
In this case the sufficient condition of SDP is within a constant factor of the information limit.
For the extreme case of $K=2$, SDP achieves the information limit with optimal constant, namely,
$\mu (1-\epsilon) \geq 2 \sqrt{\log n}$; however, in this case exact recovery can be trivially achieved by entrywise hard-thresholding or simply picking the largest entries.
\end{itemize}

\begin{figure}[hbt]
\centering
\includegraphics[width=3in]{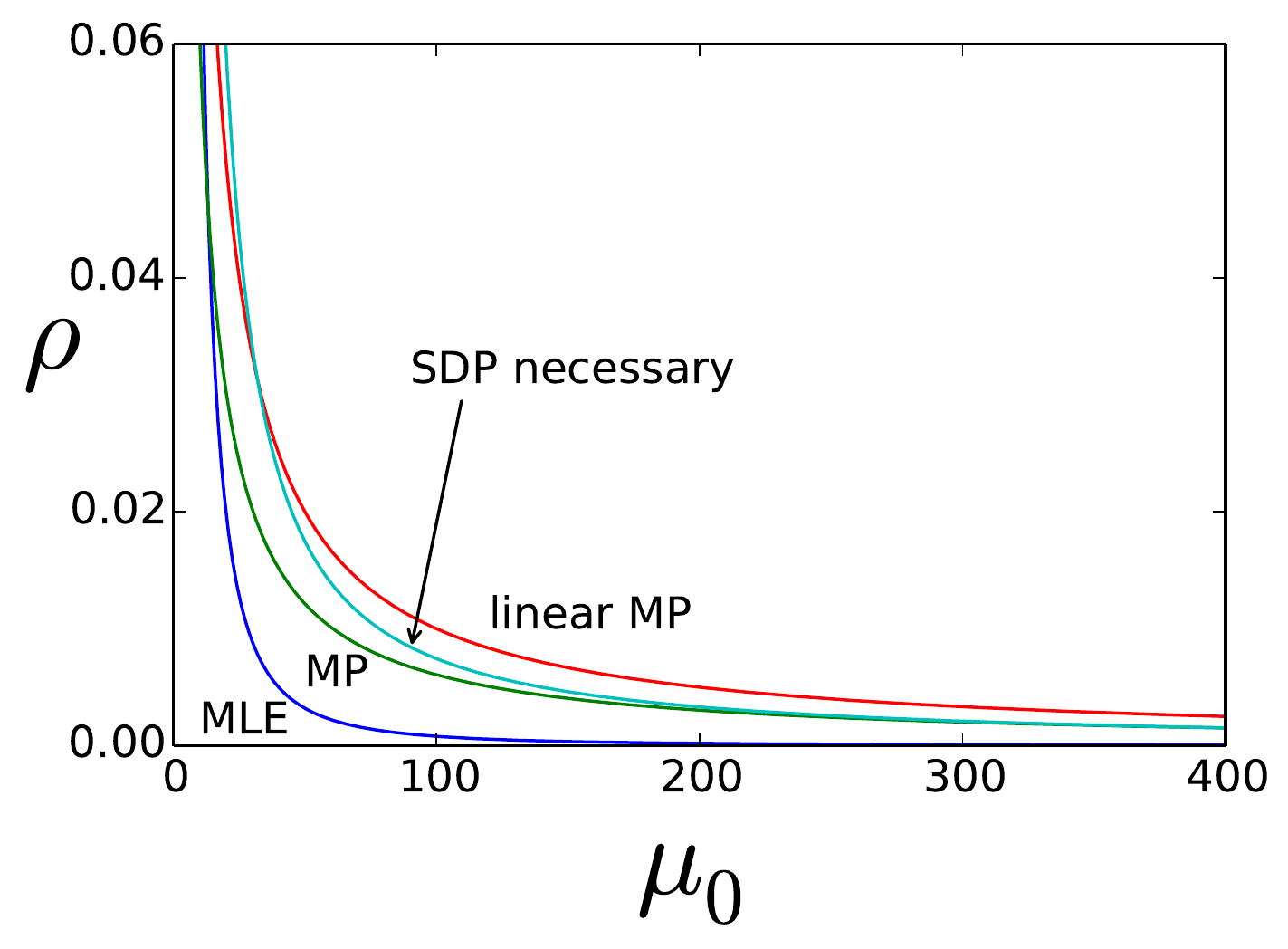}
\caption{Phase diagram for the Gaussian model with $K=\rho n /\log n$
and $\mu = \mu_0 \log n / \sqrt{n}$.   The curve MLE: $\rho \mu_0^2=1$ denotes the
information-theoretic threshold for exact recovery.
The threshold for optimized message passing (MP): $\rho^2\mu^2\eexp =1$, and
linear message passing: $\rho^2\mu^2=1$, parallel each other.   The curve
SDP necessary: $\rho\mu_0 \geq 2\sqrt{2\rho} + 1/2$ is a lower bound below
which SDP does not provide exact recovery.   The sufficient curve for SDP,
above which SDP provides exact recovery, is not shown, and lies above the
four curves shown.}
\label{fig:phase-sdp}
\end{figure}



\section{Planted densest subgraph}
In this section, we turn to the planted densest subgraph problem corresponding to the Bernoulli case of \prettyref{def:model},
where $P=\Bern(p)$ and $Q=\Bern(q)$ with $0 \le q < p \le 1$.
We prove both positive and negative results for the SDP relaxation of the MLE, \ie, \prettyref{eq:PDSCVX_SL} with $L=A$ 
being the adjacency matrix of the graph, to exactly recover the community $C^*$.

The following assumption on the  community size and graph sparsity will be imposed:
\begin{assumption}\label{ass:pds}
As $n \to \infty$, $K \to \infty$, $n-K \asymp n$,  $q$ is bounded away from $1$, and $nq =\Omega(\log n)$.
\end{assumption}

Our SDP results are in terms of the following quantities:\footnote{It can be shown that $\tau_1$ and $\tau_2$ are well-defined
whenever exact recovery is information-theoretically possible; see \prettyref{lmm:tau_1tau_2}.}
\begin{equation}
\begin{aligned}
\tau_1 = & ~ \text{solution of $K d(\tau \|p) = \log K$ in $\tau \in (0,p)$}	\\
\tau_2 = & ~ \text{solution of $K d(\tau \|q) = \log (n-K)$ in $\tau \in (q,1)$} 
\end{aligned}	
	\label{eq:tau12} 
\end{equation}

\begin{theorem} [Sufficient conditions for SDP: Bernoulli case]\label{thm:PlantedSubgraphSharp}
Suppose that \prettyref{ass:pds} holds.
If 
\begin{align}
K ( \tau_ 1 -  \tau_2) \ge \kappa \left( \sqrt{n q(1-q) } + \sqrt{ Kp(1-p) } \right),
\label{eq:sdp-suff-Bern}
\end{align}
where
\begin{align}
\kappa= 
\begin{cases}
O(1)  & nq =\Omega(\log n) \\
4+o(1) & nq = \omega(\log n)  \\
2+o(1)  &  nq=\omega(\log^4 n)
 \end{cases},
 \label{eq:kappa}
\end{align}
then $ \pprob{\widehat{Z}_{\SDP}=Z^\ast} \to 1$ as $n \to \infty$. 
\end{theorem}

\begin{remark}
To deduce sufficient condition \prettyref{eq:sdp-suff-Bern} from
the general result \prettyref{eq:SDP_suff1}, we first show  that with high probability, 
 \begin{align}
 \min_{i \in C^\ast} e(i, C^*)  &  \ge (K-1) \tau_1   \label{eq:mine_b} \\
 \max_{j \notin C^\ast} e(j, C^*) & \le K \tau_2 . \label{eq:maxe_b}
 \end{align}
 Then, we prove that with high probability,
 $$
 \|A-\expect{A}\| \le \kappa \left( \sqrt{n q(1-q) } + \sqrt{ Kp(1-p) } \right).
 $$ 
Note that $\|A-\expect{A}\|$ behaves roughly the same as $\|A'-\expect{A'}\|$, where $A'$ is the adjacency matrix of $\calG(n,q)$. In light of the concentration results for Wigner matrices and the fact that $\|A'-\expect{A'}\|=\omega_P(\sqrt{nq})$ whenever $nq = o(\log n)$ (cf.~\cite[Appendix A]{HajekWuXuSDP14}), it is reasonable to expect that $\|A'-\expect{A'}\|=\sqrt{nq}(2+o_P(1))$ whenever the average degree satisfies $nq=\Omega(\log n)$; however, this still remains an open problem cf.~\cite{LeVershynin15} and the best known upper bounds depends on the scaling 
of $nq$. This explains the piecewise expression of $\kappa$ in \prettyref{eq:kappa}.
\end{remark}

\begin{theorem} [Necessary conditions for SDP: Bernoulli case]\label{thm:PlantedSubgraphSharp_sdp_nec}
Suppose that \prettyref{ass:pds} holds, and $K=o(n)$. 
If $\liminf_{n \to \infty}  \pprob{Z^\ast = \widehat{Z}_{\SDP}} > 0$,
then 
\begin{align}
K  \geq &  \frac{1}{\kappa} \sqrt{\frac{n q}{1-q} } +1, \label{eq:sdp-necc-BernXX}  \\
K ( \tau_ 1 -  \tau_2)   \geq & \frac{1}{\kappa} \sqrt{\frac{n q}{1-q} } (1 -\tau_2) -  6 \sqrt{ \frac{K p}{\log K} } -
\frac{K (p-q) ( 2 \log \log K +1) }{\log K}  ,
\label{eq:sdp-necc-Bern}
\end{align}
where $\kappa$ is defined in \prettyref{eq:kappa}.
\end{theorem}
\begin{remark}
We prove \prettyref{eq:sdp-necc-BernXX} by contradiction: assuming \prettyref{eq:sdp-necc-BernXX} 
is violated, we construct explicitly a high-probability feasible solution $Z$ to  \prettyref{eq:PDSCVX_SL} 
based on the optimal solution of SDP defining $V_{n-K}(K)$ given in \prettyref{eq:Vm}, and show that $\Iprod{A}{Z}=\Iprod{A}{Z^*}$, contracting the unique optimality of $Z^*$. Notice that in the special case of $p=1$ (planted clique), $Z^*$ is always a maximizer of the SDP \prettyref{eq:PDSCVX_SL} therefore the failure of SDP amounts to multiple maximizers.

To deduce  the necessary condition \prettyref{eq:sdp-necc-Bern} from 
\prettyref{thm:PlantedGeneralSharp_nec}, we first establish some  inequalities similar to
 \prettyref{eq:mine_b}
and \prettyref{eq:maxe_b} but in the reverse direction.
Then, we prove a high-probability lower bound to $V_{m}(a)$
 and choose $a=\frac{1}{\kappa} \sqrt{\frac{n q}{1-q} } +1$. 
\end{remark}

\begin{remark}
\label{rmk:PC}
Particularizing \prettyref{thm:PlantedSubgraphSharp} and \prettyref{thm:PlantedSubgraphSharp_sdp_nec} to the planted clique problem ($p=1$ and $q=1/2$), we conclude that: for any fixed $\epsilon>0$, if $K \geq 2(1+\epsilon) \sqrt{n}$, then SDP succeeds (namely, $Z^*$ is the unique optimal solution to \prettyref{eq:PDSCVX_SL}) with high probability; conversely, if $K \le (1-\epsilon) \sqrt{ n}/2$, SDP fails with high probability.
In comparison, a message passing algorithm plus clean-up is shown in \cite{Deshpande12} to succeed if
$K> (1+\epsilon)\sqrt{n/\eexp}$.

\end{remark}

Assume that $\log \frac{p(1-q)}{q(1-p)}$ is bounded. If
$\frac{K(p-q)}{\sqrt{nq}} =O(1)$, 
then the sufficient condition of SDP given in \prettyref{thm:PlantedSubgraphSharp} 
reduces to $\frac{K(\tau_1-\tau_2)}{\sqrt{nq}} \ge \Omega(1)$, 
while 
the necessary condition of SDP given in  \prettyref{thm:PlantedSubgraphSharp_sdp_nec} 
reduces to $\frac{K(\tau_1-\tau_2)}{\sqrt{nq}} \ge \Omega(1)$.
Thus, the sufficient and necessary conditions are within constant factors of each other.
If instead $\frac{K(p-q)}{\sqrt{nq}} =\omega(1)$, then  
SDP attains the information-theoretic recovery threshold with sharp constants, as shown
in the next subsection.

\subsection{Comparison to the information-theoretic limits}
\label{sec:PDS-info}
In this section, we compare the performance limits of SDP with the information-theoretic limits of exact recovery obtained in \cite{HajekWuXu_one_info_lim15} under the assumption that $\log \frac{p(1-q)}{q(1-p)}$ is bounded
and $K/n$ is bounded away from $1$. 
Let
\begin{align}
\tau^\ast \triangleq \frac{ \log \frac{1-q}{1-p} + \frac{1}{K} \log \frac{n}{K} }{\log \frac{p(1-q)}{q(1-p) } }. \label{eq:deftau}
\end{align}
It is shown in  \cite[Theorem 3]{HajekWuXu_one_info_lim15} that, the optimal estimator, MLE, achieves exact recovery if
\begin{equation}
\liminf_{n \to \infty} \frac{K d(\tau^\ast \| q) }{ \log n} > 1,  \text{ and } \liminf_{n \to \infty} \frac{K d(p \| q) }{ \log (n/K) } > 2
 \label{eq:infor_exact_Bernoulli}.
\end{equation}
Conversely, if
\begin{equation}
\limsup_{n \to \infty} \frac{K d(\tau^\ast \| q) }{ \log n} < 1 ,  \text { or } \limsup_{n \to \infty} \frac{K d(p \| q) }{ \log (n/K) } < 2,
 \label{eq:infor_exact_Bernoulli_necess}
\end{equation}
no estimator can exactly recover the community with high probability.


Next we compare the SDP conditions (Theorems \ref{thm:PlantedSubgraphSharp} and \ref{thm:PlantedSubgraphSharp_sdp_nec}) to the information limit \prettyref{eq:infor_exact_Bernoulli}--\prettyref{eq:infor_exact_Bernoulli_necess}. Without loss of generality, we can
assume the MLE necessary conditions holds. Our results on the performance limits of SDP lead to  the following observations:
\begin{itemize}
\item $K= \omega( n / \log n)$. In this case, \prettyref{eq:infor_exact_Bernoulli} implies \prettyref{eq:sdp-suff-Bern}
and thus SDP attains the  information-theoretic recovery threshold with sharp constants.
To see this, note that 
\prettyref{lmm:tau_1tau_2} shows that
$\tau_1 \ge (1-\epsilon)\tau^\ast + \epsilon p$
and $\tau_2 \le (1-\epsilon) \tau^\ast + \epsilon q$ for some  small constant $\epsilon>0$. 
Moreover, 
\prettyref{lmm:divBound} and \prettyref{lmm:taupq} imply that 
\begin{align}
\frac{K d(\tau^\ast \|q) }{ \log n} \asymp \frac{K(p-q)^2 } {q \log n} = \left[ \frac{n} {K \log n} \right] \frac{K^2(p-q)^2}{ nq }. \label{eq:divergencebound}
\end{align}
Therefore, if $K=\omega(n/\log n)$, \prettyref{eq:infor_exact_Bernoulli} implies that
$K q =\Omega(\log n)$ and $K(p-q)/\sqrt{nq} \to \infty$, and consequently
\begin{align*}
K (\tau_1 - \tau_2) \ge  \epsilon K  (p-q) =\omega( \sqrt{ n q} ),
\end{align*}
which in turn implies condition \prettyref{eq:sdp-suff-Bern}.
This result recovers the previous result in the special case of $K= \rho n$, $p=a \log n/n$, and $q=b \log n/n$
with fixed constants $\rho, a, b$, where  SDP has been shown to attain   the information-theoretic   recovery threshold with sharp constants \cite{HajekWuXuSDP14}.

\item $K=o(n/\log n)$. In this case, condition  \prettyref{eq:sdp-necc-Bern} together with $q \le \tau_2 \le p $ and
$ \tau_1 \le p$
implies that $K(p-q)/\sqrt{nq}=\Omega(1)$. In comparison, in view of \prettyref{eq:divergencebound}, 
$K(p-q)/\sqrt{nq} = \omega(\sqrt{K \log n/ n} )$ is sufficient for
the information-theoretic sufficient condition \prettyref{eq:infor_exact_Bernoulli} to hold.
Hence, in this regime, SDP is order-wise suboptimal. 
\end{itemize}

The above observations imply that a gap between the performance limit of SDP and information-theoretic
 limit emerges at $K=\Theta(n/\log n)$. 
To elaborate on this, consider the following regime:
\begin{align}
K = \frac{\rho n}{\log n}, \quad p= \frac{a \log^2 n}{n},  \quad q= \frac{b \log^2 n}{n},  \label{eq:lognregime}
\end{align}
where $\rho>0$ and $a>b>0$ are fixed constants.
Let $I( x, y) \triangleq x- y \log (\eexp x/y)$ for $x,y>0$.
Let $\gamma_1$ satisfy $\gamma_1 < a$ and $ \rho I(a, \gamma_1)=1$
 and $\gamma_2$ satisfy  $ \gamma_2 > b$ and $ \rho I(b, \gamma_2) =1$.
The following corollary follows from the performance limit of MLE given by \prettyref{eq:infor_exact_Bernoulli}-\prettyref{eq:infor_exact_Bernoulli_necess} and that of SDP given by \prettyref{eq:sdp-suff-Bern}-\prettyref{eq:sdp-necc-Bern}. 
\begin{corollary}\label{cor:lognregime}
Assume the scaling \prettyref{eq:lognregime}.
\begin{itemize}
\item If $\gamma_1>\gamma_2$, then MLE attains exact recovery; conversely, 
if MLE attains exact recovery, then $\gamma_1 \ge \gamma_2$.
\item If $\rho( \gamma_1- \gamma_2)> 4 \sqrt{b} $, then SDP attains exact recovery; 
conversely, if SDP attains exact recovery, 
then $\rho ( \gamma_1-\gamma_2)  \ge \sqrt{b}/4$.
\end{itemize}
\end{corollary}
The proof is deferred to \prettyref{app:bern}. The above corollary implies that in the regime of \prettyref{eq:lognregime}, SDP is order-wise optimal, but strictly suboptimal by a constant factor.
In comparison, as shown in \cite{HajekWuXu_one_beyond_spectral15}, 
belief propagation plus clean-up succeeds if $\gamma_1 > \gamma_2$
and $\rho (a-b) > \sqrt{b/\eexp}$, while a linear message-passing algorithm 
corresponding to spectral method succeeds if $\gamma_1 > \gamma_2$
and $\rho (a-b) > \sqrt{b}$.

\section{Stochastic block model with $\Omega(\log n)$ communities}
\label{sec:sbm}
In this section, we consider the stochastic block model 
with $r\ge 2$ communities of size $K$ in a network of $n=r K$ nodes. 
Derived in \cite{HajekWuXuSDP15,ABBK,perry2015semidefinite}, the
following SDP is a natural convex relaxation of MLE:\footnote{There are slightly different
but equivalent ways to impose the constraints.  Under the
condition $Y \succeq 0,$   the constraint    $\Iprod{Y}{\allones} =0$
is equivalent to $Y \ones =0$, which is the formulation used in \cite{HajekWuXuSDP15}.}
\begin{align}
\hat{Y}_{\SDP} = \argmax_{Y \in \reals^{n\times n} }  & \; \Iprod{A}{Y} \nonumber  \\
\text{s.t.	} & \; Y \succeq 0, \nonumber  \\
& \;  Y_{ii} =1, \quad i \in [n] \nonumber \\
& \;  Y_{ij} \ge -\frac{1}{r-1}, \quad i, j \in [n]  \nonumber \\
& \;  \Iprod{Y}{\allones} =0. \label{eq:SBM_multiple}
\end{align}
Define the $n\times n$ symmetric matrix $Y^\ast$ corresponding to the true clusters 
by $Y^\ast_{ij} =1$ if vertices $i$ and $j$ are in the same cluster, including the case $i=j$,
and $Y^\ast_{ij}=- \frac{1}{r-1}$ otherwise. 

Consider $p= \frac{\alpha \log n}{K}$  and $q= \frac{\beta \log n}{K}$ for
fixed constants $\alpha > \beta > 0$.
For constant number of communities, namely $r=O(1)$,
the sharp optimality of SDP has been established in \cite{HajekWuXuSDP15}:
if $\sqrt{\alpha}-\sqrt{\beta}>1$,
$\hat{Y}_{\SDP} = Y^\ast$ with high probability; conversely,
if $\sqrt{\alpha}-\sqrt{\beta}<1$  and 
 the clusters are uniformly chosen at random
among all $r$-equal-sized partitions of $[n]$,
then for any sequence of estimators $\widehat{Y}_n$,
$\pprob{\widehat{Y}_n=Y^\ast} \to 0$ as $n \to \infty$.
The optimality of SDP has been extended to $r = o(\log n)$ communities
in \cite{ABBK}.   Determining whether SDP continues
to be optimal for $r = \Omega(\log n),$ or equivalently, for communities
of  size $K=O( \frac{n}{\log n})$, is left as an open question in \cite{ABBK}. 
Next, we settle this question by proving that in contrast to the MLE, SDP 
is constantwise suboptimal when $r \geq C \log n$ for sufficiently large $C$,
and orderwise suboptimal when $r \gg \log n$. 
What remains open is to assert the suboptimality of SDP for \emph{all} $r = \Theta(\log n)$ similar to the single-community case.


\begin{theorem}
\label{thm:SBM_sdp_nec}
Suppose  $p = o(1)$, $q=\Theta(p),$ and $r\to \infty.$ 
If  $\liminf_{n \to \infty}  \pprob{\widehat{Y}_{\SDP}=Y^\ast} > 0$,
then
\begin{equation}
	K(p-q)^2\geq  \frac{rq^2(1+o(1))}{p\kappa^2},
	\label{eq:sdp-necc-Bern_SBM}
\end{equation}
where $\kappa$ is the constant defined in \prettyref{eq:kappa}.
\end{theorem}
\begin{proof}
	\prettyref{sec:pf-sbm}.
\end{proof}


\begin{remark}  \label{rmk:ChenXu}
Under the assumption of $q =\Theta(p)$, the information-theoretic condition has been established in \cite{ChenXu14}: MLE succeeds with high probability if and only if 
\begin{equation}
K (p-q)^2 \asymp q \log n.
	\label{eq:chenxu}
\end{equation}
Comparing \prettyref{eq:chenxu} to the necessary condition \prettyref{eq:sdp-necc-Bern_SBM} for SDP, 
we immediately conclude that SDP 
is orderwise suboptimal if $r=\omega(\log n)$, or equivalently, $K=o( \frac{n}{\log n})$.
Furthermore, if $r \geq C \log n$ for a sufficiently large constant $C$, SDP is suboptimal in terms of constants, which  is consistent with the single-community result in \prettyref{sec:main}.
\end{remark}

\section{Discussions}

%

In this paper, we derive a sufficient condition and a necessary condition for the
success of an SDP relaxation \prettyref{eq:PDSCVX_SL} for exact recovery under the general $P/Q$ model. For both the Gaussian
and Bernoulli cases, the general results imply that the SDP
attains the information-theoretic recovery limits with sharp constants if and only
if $K=\omega(n/\log n)$. 
Loosely speaking, there are two types of perturbation which can lead to a higher objective value and 
prevent the true cluster matrix $Z^*$ being the unique maximizer of the SDP. One is the  \emph{local} perturbation of 
the ground truth corresponding to swapping a node in the community with one
outside. In order for exact recovery to be informationally possible, 
the optimal estimator, MLE, must also remain insensitive to this local perturbation.
The other is the \emph{global} perturbation induced by the solution of the
auxiliary SDP \prettyref{eq:Vm}. This global perturbation is closely related to the 
spectral perturbation, \ie, $\|A-\expect{A}\|$, which is responsible for the suboptimality of 
the spectral algorithms. It turns out that when $K=\omega(n/\log n)$, the local perturbation
dominates the global one, leading to the attainability of the optimal threshold by SDP;
however, when $K=O(n/\log n)$, the local perturbation is dominated by the global one, 
resulting in the suboptimality of SDP.

An interesting future direction is to establish upper and lower bounds of SOS relaxations
for the problem of finding a hidden community in relatively sparse SBM. 

\section{Proofs }
In this section, we prove our main theorems. 
In particular, \prettyref{sec:proof_suff} contains the proofs of SDP sufficient conditions
given in \prettyref{thm:PlantedGeneralSharp}, \prettyref{thm:PlantedSharp_SL}, 
and \prettyref{thm:PlantedSubgraphSharp}. 
The proofs of SDP necessary conditions given in \prettyref{thm:PlantedGeneralSharp_nec},
\prettyref{thm:sdp-nece},  and \prettyref{thm:PlantedSubgraphSharp_sdp_nec}
are presented in \prettyref{sec:proof_nec}.

\subsection{Sufficient Conditions}\label{sec:proof_suff}
In this subsection, we provide the proof of \prettyref{thm:PlantedGeneralSharp}, as well as
the proofs of its further consequence
in the Gaussian and Bernoulli cases.

Before the main proofs, we need a dual certificate lemma,  providing a set of deterministic conditions which is both sufficient and necessary
for the success of SDP  \prettyref{eq:PDSCVX_SL}.
\begin{lemma}\label{lmm:PlanteddenseKKT}
$Z^\ast$ is an optimal solution to \prettyref{eq:PDSCVX_SL} if and only if the following KKT conditions hold:
there exist $D=\diag{d_i} \ge 0$, $B \in \calS^n$ with $B \ge 0$, $\lambda, \eta \in \reals$ such that $ S \triangleq D-B- L +  \eta \identity + \lambda \allones$ satisfies $S \succeq 0$, and
\begin{align}
S \xi^\ast &=0, \label{eq:KKT1}\\
d_i(Z^\ast_{ii}-1) & =0, \quad \forall i, \label{eq:KKT2}\\
B_{ij} Z^\ast_{ij} &=0, \quad \forall i,j. \label{eq:KKT3}
\end{align}
If further
 \begin{align}
 \lambda_2(S)&>0,  \label{eq:KKT11}
 \end{align}
 or
 \begin{align}
\min_{i \in C^*} d_i >0,  \text{ and } \min_{ (i,j) \notin C^* \times C^*: i \neq j} B_{ij} >0, \label{eq:KKT22}
 \end{align}
 then $Z^\ast$ is the unique  optimal solution to \prettyref{eq:PDSCVX_SL}.
\end{lemma}
\begin{proof}
Notice that
$Z=  \frac{K(n-K)}{n(n-1)}\identity+ \frac{K(K-1)}{n(n-1)} \allones $ is strictly feasible to \prettyref{eq:PDSCVX_SL}, \ie,
the Slater's condition holds, which implies, via Slater's theorem for SDP, that strong duality holds (see, \eg, \cite[Section 5.9.1]{Boyd2004}). Thus the KKT conditions given in \prettyref{eq:KKT1}--\prettyref{eq:KKT3} are both sufficient and necessary
for the optimality of $Z^\ast$.

To show the uniqueness of $Z^\ast$ under condition \prettyref{eq:KKT11} or condition \prettyref{eq:KKT22},
consider  another optimal solution $\tZ$. Then,
\begin{align*}
\Iprod{S}{\tZ} & =\Iprod{D-B-L+ \eta \identity + \lambda \allones }{\tZ} \overset{(a)}{=}\Iprod{D-B-L}{\tZ} + \eta K + \lambda K^2 \\
 & \overset{(b)}{\le} \Iprod{D-L}{Z^\ast} +\eta K + \lambda K^2  {=}\Iprod{S}{Z^\ast} =0.
\end{align*}
where $(a)$ holds because $\Iprod{\identity}{\tZ}=K$ and $\Iprod{\allones}{\tZ}=K^2$;
$(b)$ holds because $\Iprod{L}{\tZ}=\Iprod{L}{Z^\ast}$, $B,\tZ \ge 0$, and $\Iprod{D}{\tZ} \leq \sum_{i\in C^*} d_{i} = \Iprod{D}{Z^*}$ in view of $d_i \geq 0$ and $\tZ_{ii} \leq 1$ for all $i \in [n]$. It follows that the inequality $(b)$ holds with equality, and thus $\Iprod{D}{ \tilde{Z} -Z^\ast}=0$ and $\Iprod{B}{\tilde{Z}} =0$.

Suppose \prettyref{eq:KKT11} holds. Since $\tZ \succeq 0$,  $S \succeq 0$, and $\Iprod{S}{\tZ}=0$,
$\tZ$ needs to be a multiple of $Z^\ast=\xi^\ast (\xi^\ast)^\top$.
Then $\tZ=Z^\ast$ since $\Tr(\tZ)=\Tr(Z^\ast)=K$.

Suppose instead  \prettyref{eq:KKT22} holds.  Since $\Iprod{B}{\tilde{Z}} =0$ and $B,\tZ \ge 0$,
it follows that $\tilde{Z}_{ij}=0$ for all $i \neq j$ such that $(i,j) \notin C^* \times C^\ast$. Also, in view of
$\Iprod{D}{ \tilde{Z} -Z^\ast}=0$ and $\tZ_{ii} \leq 1$, we have that $\tZ_{ii}=1$ for all $i \in C^*$.
Hence, $\tZ_{ii}=0$ for all $i \notin C^\ast$ due to $\Iprod{\identity}{\tZ}=K$. Finally, it follows from $\Iprod{\allones}{\tZ}=K^2$ that $\tZ_{ij}=1$ for all $(i,j) \in C^\ast \times C^\ast$. Hence, we conclude that
$\tZ=Z^\ast$.
\end{proof}

\begin{proof}[Proof of \prettyref{thm:PlantedGeneralSharp}]
We construct $(\lambda, \eta, S, D, B)$ which
satisfy the conditions in \prettyref{lmm:PlanteddenseKKT}. 
Observe that to satisfy \prettyref{eq:KKT1}, \prettyref{eq:KKT2}, and \prettyref{eq:KKT3},
we need that $D=\diag{d_i}$ with
\begin{align}
d_i=\left\{
 \begin{array}{rl}
   \sum_{j \in C^\ast} L_{ij} - \eta - \lambda K   & \text{if } i \in C^\ast\\
   0 & \text{otherwise}
    \end{array} \right., \label{eq:positivityD}
\end{align}
and $B_{ij} = 0$ for $i, j \in C^\ast$, and
\begin{equation}
	\sum_{j \in C^\ast} B_{ij} = \lambda K - \sum_{j \in C^\ast } L_{ij},  \quad \forall i \notin C^\ast,\label{eq:positivityB}
\end{equation}
where, given $\lambda$,  $\eta$  can be chosen without loss of generality to be:
\begin{align*}
\eta  & = \min_{i \in C^\ast} e(i, C^*) - \lambda K.
\end{align*}
There remains flexibility in the choice of $\lambda$ and the completion
of the specification of $B$.   
Recall that $\alpha= \mathbb{E}_P[ L_{12} ]$ and $\beta=\mathbb{E}_Q[ L_{12} ]$.
We let
\begin{align*}
\lambda &=  \max \left\{ \max_{i \notin C^\ast } e(i, C^*) /K,  \beta \right\}  \\
B_{ij}  & = b_i \indc{i \notin C^*, j \in C^\ast} + b_j \indc{i \in C^*, j \notin C^\ast},
\end{align*}
where
$b_i \triangleq \lambda - \frac{1}{K}\sum_{j \in C^\ast} L_{ij}$ for $i \notin C^\ast$.
By definition, we have $d_i(Z^\ast_{ii}-1) =0$
and $B_{ij} Z^\ast_{ij} =0$ for all $i,j \in [n]$. Moreover, for all $i \in C^\ast$,
\begin{align*}
d_i \xi^\ast_i = d_i =\sum_{j} L_{ij} \xi^\ast_j  - \eta - \lambda K
= \sum_{j} L_{ij} \xi^\ast_j + \sum_j B_{ij} \xi^\ast_j - \eta - \lambda K,
\end{align*}
where the last equality holds because $B_{ij}=0$ if $(i,j) \in C^\ast \times C^\ast$;
for all $i \notin C^\ast$,
\begin{align*}
\sum_j  L_{ij}  \xi^\ast_j + \sum_j   B_{ij}\xi^\ast_j   - \lambda K = \sum_{j \in C^\ast}  L_{ij} + K b_i -\lambda K =0,
\end{align*}
where the last equality follows from our choice of $b_i$. Hence, $D \xi^\ast=L \xi^\ast + B \xi^\ast - \eta \xi^\ast - \lambda K \mathbf{1}$ and consequently $S \xi^\ast =0$.  Also,
by definition, $\min_{i \in C^\ast} d_i \ge 0$ and $\min_{i \notin C^\ast} b_i \ge 0$,
and thus  $D \ge 0$, $B \ge 0$.

It remains to verify $S \succeq 0$ with $\lambda_2(S)>0$,  \ie,
\begin{equation}
\inf_{x \Perp \xi^\ast, \|x\|_2=1} x^\top S x  > 0.	
	\label{eq:lambda2PDS}
\end{equation}
Since 
\begin{align*}
\expect{L}=  \left( \alpha+\beta \right) Z^\ast  + \beta \allones- \alpha
\begin{bmatrix}
 \identity_{K \times K} & \zeros \\
\zeros & \zeros
\end{bmatrix}
 - \beta \begin{bmatrix}
  \zeros & \zeros \\
\zeros & \identity_{(n-K) \times (n-K)}
\end{bmatrix}  , 
\end{align*}
it follows that for any $x \Perp \xi^\ast$ and $\|x\|_2=1$,
\begin{align}
x^\top S x & ~ = x^\top D x - x^\top B x +  (\lambda - \beta) x^\top \allones x + \alpha \sum_{i \in C^\ast} x_i^2  + \beta \sum_{i \notin C^\ast} x_i^2
+ \eta - x^\top \left(L -\expect{L} \right) x
\nonumber  \\
& ~ \overset{(a)}{=} \sum_{ i \in C^\ast}  d_i   x_i^2  + (\lambda - \beta) x^\top \allones x 
+ \alpha \sum_{i \in C^\ast} x_i^2
 + \beta  \sum_{i \notin C^\ast} x_i^2 + \eta -x^\top \left(L -\expect{L} \right) x  \nonumber \\
& ~ \overset{(b)}{\ge}   \min_{ i \in C^\ast} \sum_{i \in C^\ast}   d_i  x_i^2 + (\lambda - \beta) x^\top \allones x
   + \beta  + \eta
 - \| L -\expect{L}\| \nonumber \\
 &~  \overset{(c)}{>}  \left( \min_{ i \in C^\ast} d_i  \right)  \sum_{i \in C^\ast} x_i^2 \ge 0.
  \label{eq:semidefinitebound}
 \end{align}
where $(a)$ holds because $B_{ij}=0$ for all $(i, j) \in C^\ast \times C^\ast$ and
\begin{align*}
x^\top B x  = 2 \sum_{i \notin C^\ast} \sum_{j \in C^\ast} x_i x_j B_{ij}=
2 \sum_{i \notin C^\ast} x_i b_{i} \sum_{j \in C^\ast} x_j =0 ;
\end{align*}
$(b)$ follows due to the assumption that $\alpha \ge \beta$ and the fact that $ x^\top (L-\expect{L} ) x  \le \|L -\expect{L} \|$;
$(c)$ holds because by assumption  $\eta > \| L -\expect{L}\| - \beta$ and $\lambda   \ge \beta $;
the last inequality follows due to $ \min_{i \in C^\ast} d_i \ge 0$.
Therefore, the desired  \prettyref{eq:lambda2PDS} holds in view of \prettyref{eq:semidefinitebound}, completing the proof.
\end{proof}

\subsubsection{Gaussian case}

We need the following standard result in extreme value theory (\eg,  see \cite[Example 10.5.3]{DN70} and use union bound).
\begin{lemma}
Let $\{Z_i\}$ be a sequence of standard normal random variables. Then
\[
\max_{i \in [m]} Z_i \leq  \sqrt{2 \log m} +o_P(1),
\quad m \diverge,
\]
with equality if the random variables are independent.
	\label{lmm:extreme}
\end{lemma}

\begin{proof}[Proof of \prettyref{thm:PlantedSharp_SL}]
In the Gaussian case, $\mathbb{E}_P[A_{12}] =\mu$  and $\mathbb{E}_Q[A_{12}] =0$. Hence, 
in view of \prettyref{thm:PlantedGeneralSharp}, it suffices to show that
with probability tending to one,
\begin{align}
    \min_{i \in C^\ast}  \sum_{j \in C^*} A_{ij}  - 
    \max \left\{  \max_{i \notin C^\ast} \sum_{j \in C^*} A_{ij}  , 0 \right \}  & > \| A-\expect{A} \|    \label{eq:SDP_suff1_g}. 
\end{align}
By \prettyref{lmm:extreme},
\begin{align*}
  \max_{i \notin C^\ast}  \sum_{j \in C^*} A_{ij}  \le \sqrt{2 K \log(n- K)}  + o_P(\sqrt{K}).
   \end{align*}
Note that  $ \left \{ \sum_{j \in C^*} A_{ij} : i \in C^\ast \right \}$ are not mutually independent.
By \prettyref{lmm:extreme} applied to $-A_{ij},$
\begin{align*}
 \min_{i \in C^\ast}  \sum_{j \in C^*} A_{ij}  \ge (K-1) \mu -  \sqrt{2 (K-1) \log K} + o_p(\sqrt{K}) .
\end{align*}
By \prettyref{lmm:GDS},
for any sequence $t_n \to \infty,$
$$
\norm{A-\expect{A}} \leq 2\sqrt{n}+t_n
$$ with probability converging to one.
Hence,  in view of the assumption \prettyref{eq:sdp-suff}, we have that  \prettyref{eq:SDP_suff1_g} holds with high probability. 

In the remainder, we prove \prettyref{eq:SDP-trivial} for any $K \geq 2$ implies that $Z^*$ is the unique
optimal solution of the SDP.
We write $T= \{ (i,j) \in C^\ast \times C^\ast : i \neq j  \}$ and $T^c = \{ (i,j) \in [n] \times [n]: i \neq j  \} \backslash T$.
Recall that for distinct $i,j$, $A_{ij} \sim \calN(\mu,1)$ if $i,j\in C^*$ and $\calN(0,1)$ otherwise.
Using \prettyref{lmm:extreme} and the assumption \prettyref{eq:SDP-trivial},  we have
\begin{equation}
\min_{ (i,j) \in T } A_{ij} > \max_{ (i,j) \in T^c } A_{ij}	
	\label{eq:minmax}
\end{equation}
 with probability converging to $1$. Hence, without loss of generality,  we can and do assume that \prettyref{eq:minmax} holds  in the following.  Let $Z$ be any feasible solution of SDP \prettyref{eq:PDSCVX_SL}.
Since  $Z_{ii} \leq 1$ for all $i$ and $Z \succeq 0,$ it follows that $|Z_{ij}| \leq 1$ for all $i,j$.
Hence $0\le Z \le \allones$.    Also, $\Iprod{\allones-\identity}{Z} = K(K-1)$.   So
$\Iprod{Z}{A}$ is a weighted sum of the terms $(A_{ij}: i\neq j),$  where the weights
$Z_{ij}$ are nonnegative, with values in [0,1], and total weight equal to $K(K-1)$.   The sum is thus
maximized if and only if all the weight is placed on the $K(K-1)$ largest terms, namely $A_{ij}$ with $(i,j)\in T$,
which are each strictly larger than the other terms.  Thus, $Z^*$ is the unique maximizer.

\end{proof}

\subsubsection{Bernoulli case}

\begin{proof}[ Proof of \prettyref{thm:PlantedSubgraphSharp}]
In  the Bernoulli case, $\mathbb{E}_P[A_{12}] =p$ and $\mathbb{E}_Q[A_{12}] =q$.  Hence, in view of \prettyref{thm:PlantedGeneralSharp},  it reduces to
show that with probability tending to one,
\begin{align}
    \min_{i \in C^\ast}  \sum_{j \in C^*} A_{ij}  - 
    \max \left\{ \max_{i \notin C^\ast} \sum_{j \in C^*} A_{ij} , Kq \right \}   & > \| A-\expect{A} \|  - q  \label{eq:SDP_suff1_b}. 
\end{align}
We will use the following upper bounds for the binomial distribution tails \cite[Theorem 1]{ZubkovSerov13}:
\begin{align}
\prob{ \Binom(m ,p) \le m \tau -1 } &\le Q\left( \sqrt{2 m d ( \tau \| p ) } \right), \quad 2/m \le \tau  \le p, \\
\prob{ \Binom(m, q) \ge m \tau + 1 } & \le Q\left( \sqrt{ 2m d( \tau \| q) } \right), \quad  q \le \tau \le 1-1/m,
\end{align}
where $Q(\cdot)$ denotes the standard normal tail probability.
By the definition of $\tau_1$ and $\tau_2$, it follows that
\begin{align*}
\prob{ \sum_{j \in C^*} A_{ij}    \leq   (K-1) \tau_1 -1 }  &\le Q ( \sqrt{ 2 (K-1) \log K / K} ) = o(1/K), \quad  \forall i \in C^*\\
\prob{  \sum_{j \in C^*} A_{ij }     \geq   K \tau_2  + 1 }  &\le Q ( \sqrt{ 2 \log (n-K) } ) = o(1/(n-K)), \quad \forall i \notin C^*.
\end{align*}
By the union bound,  with high probability, 
\begin{align*}
  \min_{i \in C^\ast}  \sum_{j \in C^*} A_{ij}   & >  (K-1) \tau_1 -1 \\
 \max_{i \notin C^\ast} \sum_{j \in C^*} A_{ij}  & < K \tau_2  + 1.
\end{align*}

We decompose $A=A_1+A_2$, where $A_1$ is obtained from $A$ by setting all entries not in $C^\ast \times C^\ast$ to be zero; similar, $A_2$ is obtained from $A$
by setting all entries in $C^\ast \times C^\ast$ to be zero. 
Applying \prettyref{lmm:concen_bern} yields that 
with high probability, 
\begin{align*}
\| A - \expect{A} \| & \le \| A_1 - \expect{A_1} \| + \| A_2 - \expect{A_2} \| \\
& \le \kappa  \left( \sqrt{ Kp(1-p) } + \sqrt{ n q(1-q)} \right),
\end{align*}
where $\kappa$ is defined in \prettyref{eq:kappa}. 
Hence, in view of the assumption \prettyref{eq:sdp-suff-Bern}, we have that \prettyref{eq:SDP_suff1_b} holds with high probability.
\end{proof}


\subsection{Necessary conditions} \label{sec:proof_necc}
\label{sec:proof_nec}

\begin{proof}[Proof of \prettyref{thm:PlantedGeneralSharp_nec}]
The proof is a slight variation of the heuristic derivation
given before the statement of \prettyref{thm:PlantedGeneralSharp_nec}.
Fix $K, n, C^*,$ the matrix $L$, and a constant
$a$ with $1\leq a \leq K$ and let $r=\frac{a}{K}$.
Suppose the indices are ordered and the matrix $U$
is defined as in the heuristic derivation.

Let $Z$ be defined as a function of $\epsilon \geq 0$ as follows.   We shall specify
 $\alpha$ and $\beta$ depending on $\epsilon$ for sufficiently small $\epsilon$
in such a way that
 \begin{align}   \label{eq:alpha_beta}
 \alpha \leq 1,~  \alpha = 1+O(\epsilon^2),~  \beta \geq 1-r, ~ \beta=1-r + O(\epsilon) .
 \end{align}
 Let $\xi_{\epsilon}$ be the column vector with $K+1$ nonzero entries, defined by
$\xi_\epsilon = (1, \dots  , 1, 1-\epsilon, \beta\epsilon, 0 , \ldots , 0)^\top$.
Finally,  let $Z=\alpha \xi_{\epsilon}  \xi_{\epsilon}^T + 2\epsilon U$.   In expanded form:
\begin{align*}
Z & = \alpha \left[
    \begin{array}{cccc;{2pt/2pt}cc}
        1  & \cdots &    1      &  1-\epsilon  &  \beta\epsilon & \hspace{1in} \\
  \vdots & \vdots & \vdots & \vdots & \vdots   \\
        1  & \cdots &    1      &  1-\epsilon  &  \beta\epsilon  \\  
        1-\epsilon  & \cdots &    1-\epsilon      &  (1-\epsilon)^2  &  \beta\epsilon(1-\epsilon)  \\ \hdashline[2pt/2pt]
 \beta\epsilon & \cdots &   \beta\epsilon   & \beta\epsilon(1-\epsilon)  & \beta^2\epsilon^2  \\
&&&&\\
&&&&\\
&&&&\\
    \end{array}
\right]   + 2\epsilon U
\end{align*}
Up to $o(\epsilon)$ terms, $Z$ is equal to the matrix $Z^* + \delta_1 + \delta_2 + \delta_3$ described in the
heuristic derivation.
Clearly for  $\epsilon$ sufficiently small,  $Z\geq 0,$  $Z\succeq 0,$ and $Z_{ii} \leq 1$.
It is also straightforward to see that
$$
\Iprod{L}{Z-Z^*} = 2\epsilon \left( (1-r)  \max_{i\not\in C^*} e(i,C^*) + V_{n-K}(a)  - \min_{i\not\in C^*} e(i,C^*)  \right) + o(\epsilon),
$$
so that once we establish the feasibility of $Z,$ the proof will be complete.   That is, it remains to
show that $\alpha$ and $\beta$ can be selected for sufficiently small $\epsilon$ so that
\eqref{eq:alpha_beta}, $\Iprod{\identity}{Z}= K$, and  $\Iprod{\allones}{Z}= K^2$ hold true.
The later two equations can be written as
\begin{align}
\alpha\left\{   K -2\epsilon  + (1+\beta^2)\epsilon^2  \right\}  & = K-2\epsilon   \label{eq:alphabetaI}  \\
\alpha \left\{ K - (1-\beta)\epsilon \right\}^2 & = K^2 - 2\epsilon K r   \ \label{eq:alphabetaII}.
\end{align}
Combining \eqref{eq:alphabetaI} and \eqref{eq:alphabetaII} to eliminate $\alpha$ and simplifying
yields the following equation for $\beta:$
$$
K^2(1-\beta -r) + \epsilon K (\beta -2(1-\beta -r)) + \epsilon^2\left( (1-\beta^2) -Kr(1+\beta^2) \right)=0
$$
This equation has the form $F(\epsilon,\beta)=0$ ($K$ and $r$ are fixed)
with a solution at $(\epsilon,\beta)=(0,1-r)$.   
Also, $\frac{\partial F}{\partial \epsilon}(0,1-r)= K(1-r)$
and $\frac{\partial F}{\partial \beta}(0,1-r)= -K^2\neq 0$.
Therefore, by the implicit function
theorem, the equation determines $\beta$ as a continuously differentiable function
of $\epsilon$ for small enough epsilon, and
$$
\beta= (1-r)\left(1+\frac{\epsilon}{K} \right)  + O(\epsilon^2).
$$
This expression for $\beta$ together with \eqref{eq:alphabetaI} yields that for sufficiently small $\epsilon,$   $\alpha < 1$ and
\[
\alpha =  1 - \epsilon^2\left( \frac{1+(1-r)^2}{K}\right) + O(\epsilon^3).
\]
\end{proof}

\begin{proof}[Alternative proof of \prettyref{thm:PlantedGeneralSharp_nec}]
Here is an alternative proof of \prettyref{thm:PlantedGeneralSharp_nec} via a dual-based approach. 
If $Z^*=\xi^\ast(\xi^\ast)^\top$ maximizes \prettyref{eq:PDSCVX_SL}, then by \prettyref{lmm:PlanteddenseKKT} there exist
dual variables $(S,D,B,\lambda,\eta)$ with $S=D- B- L + \eta \identity + \lambda \allones  \succeq 0$,
$B \geq 0$, $D = \diag{d_i} \geq 0$, such that \prettyref{eq:KKT1},  \prettyref{eq:KKT2} and \prettyref{eq:KKT3} are satisfied.
As a consequence, the choice of $D$ is fixed, namely,
\begin{align}
d_i=\left\{
 \begin{array}{rl}
   \sum_{j \in C^\ast } L_{ij} - \eta - \lambda K   & \text{if } i \in C^\ast \\
   0 & \text{otherwise}
    \end{array} \right..
    \label{eq:di2}
\end{align}
Therefore, the condition $\min_{i \in C^\ast } d_i \geq 0$ implies that
\begin{align}
\min_{i \in C^*} e(i, C^\ast) \ge \lambda K + \eta.  \label{eq:etaupperbound}
\end{align}
Moreover, the dual variable $B$ satisfies $B_{C^\ast C^\ast } = 0$ and the off-diagonal block $B_{(C^\ast)^c C^\ast }$ satisfies
\begin{equation}
	\sum_{j \in C^*} B_{ij} = \lambda K - \sum_{j \in C^*} L_{ij},  \quad \forall i \notin C^*.
	\label{eq:Bavg1}
\end{equation}
Denote all possible choices of $B$ by the following convex set:
\[
\calB = \{B: B \in \calS^{n}, B \geq 0, B_{C^\ast C^\ast } = 0, B_{(C^\ast)^c C^\ast } \text{ satisfies }\prettyref{eq:Bavg1} \}.
\]
In particular, we have $\sum_{j \in C^*} B_{ij} \geq 0$ for all $i \notin C^*$, which  implies that
\begin{align}
\lambda K \geq \max_{i \notin C^\ast }  e(i, C^*).  \label{eq:lambdalowerbound}
\end{align}

Finally, $S = D+\lambda \allones - B - L + \eta \identity \succeq 0$ and $S \xi^\ast =0$
imply that there exists $B \in \calB$ and $\eta$ such that $\eta \geq \sup_{ \|x\| = 1}  x^\top(B+ L-D-\lambda \allones) x$ and
$\eqref{eq:etaupperbound}$ holds.
Hence,
\begin{align}
\eta \ge & ~ \inf_{B \in \calB} \sup_{ \|x\| = 1}  x^\top(B+ L-D-\lambda \allones) x 	\nonumber   \\
= & ~ \inf_{B \in \calB} \lambda_{\max} \left( B+ L-D-\lambda \allones \right) \nonumber  	\\
\ge & ~ \inf_{B \ge 0} \lambda_{\max} \left( B+ L -D-\lambda \allones \right)  \nonumber    \\
=  & ~ \inf_{B \ge 0}~~~~ \sup_{U\succeq 0, \Iprod{U}{\identity}=1}  \Iprod{L-D-\lambda \allones +B}{U}  \nonumber   \\
\overset{(a)}{=} & ~ \sup_{U\succeq 0, \Iprod{U}{\identity}=1}      \inf_{B \ge 0} \Iprod{L-D-\lambda\allones +B}{U} \label{eq:dual}  \\
= & ~ \sup_{U \ge 0, U \succeq 0, \Iprod{U}{\identity} =1 } \Iprod{L-D- \lambda \allones}{U}, \label{eq:etalowerboundsubmatrix}
\end{align}
where (a) follows because $U=(1/n) \identity + \allones$ is strictly feasible for the supremum in \prettyref{eq:dual} (i.e. it satisfies
 Slater's condition) so the strong duality holds.

Restricting $U$ in \prettyref{eq:etalowerboundsubmatrix} to satisfy $U_{ij}=0$
except for those $i, j \notin C^\ast$, and $\Iprod{U}{\allones}= a \in [1,K]$,
we get that  $\eta  \ge  \sup_{1 \le a \le K } \left\{  V_{ n-K}(a) - a \lambda \right\}$.
It follows from \prettyref{eq:etaupperbound} that
\begin{align*}
\min_{i \in C^*} e(i, C^\ast) & \ge  \sup_{1 \le a \le K } \left\{  V_{ n-K}(a) - a \lambda \right\} + \lambda K  \\
& \ge   \sup_{1 \le a \le K   } \left\{  V_{ n-K}(a) - \frac{a}{K}   \max_{i \notin C^*}  e(i, C^*)  \right\} +  \max_{i \notin C^*}  e(i, C^*),
\end{align*}
where the last inequality follows from $a \le K$ and \prettyref{eq:lambdalowerbound}.
\end{proof}

\subsubsection{Gaussian case}
Consider the Gaussian case $P = \calN(\mu,1)$ and $Q = \calN(0,1)$.
Before the proof of \prettyref{thm:sdp-nece}, we need to introduce a key lemma to lower bound
the value of $V_m(a)$ given in \prettyref{eq:Vm}. Recall that $m=n-K$. 
By the assumption, $L=A$ and hence $M$ has the same distribution as an $m \times m$ symmetric random matrix $W$ with zero-diagonal and 
$W_{ij} \iiddistr \calN(0,1)$ for $1 \le i<j \le m$.
The following lemma provides a high-probability  lower bound on $V_m(a)$ defined in \prettyref{eq:Vm}; its proof is deferred to \prettyref{app:proofVmGaussian}.
\begin{lemma}
Assume that $a > 1$ and  $a = o(m)$ as $m\diverge$.
Let $M=W$ be an $m \times m$ symmetric random matrix with zero-diagonal and
independent standard normal entries in the definition of $V_m(a)$ in \prettyref{eq:Vm}.
Then with probability tending to one,
\begin{equation}
	V_m(a) \geq \begin{cases}
	\frac{\sqrt{m}}{2} -  r & a= \omega(\sqrt{m}) \\
	a \sqrt{\log\pth{1 + \frac{m}{4a^2}}} - o(a) & a= \Theta(\sqrt{m}) \\	
	(a-1) \sqrt{ \frac{1}{3} \log \frac{ m}{a^2} } - O(a \log \log \frac{ m}{a^2} ) & a= o(\sqrt{m})
	\end{cases}
	\label{eq:max}
\end{equation}
where  $r \triangleq  \frac{ m^{3/4}}{\sqrt{8(a-1)} }  
+ \frac{2a}{\sqrt{m}}  = o(\sqrt{m})$ if $a= \omega(\sqrt{m})$.
	\label{lmm:max}
\end{lemma}

\begin{remark}
We also have the following simple observations on $V_m(a)$:
\begin{itemize}
\item $V_m(1)=0$.
	\item
	Dropping the second and the last constraints in \prettyref{eq:Vm}, we have $V_m(a) \leq \lambda_{\max}(W) = 2 \sqrt{m}(1+o_P(1))$.
\item Since $\|W\|_{\ell_\infty}=\sqrt{2\log \binom{m}{2} } + o_P(1)$, it follows that $V_m(a) \le (a-1) \|W\|_\infty  = (a-1) \sqrt{2\log \binom{m}{2}} + o_P(a)$.
\end{itemize}
\end{remark}

We next prove \prettyref{thm:sdp-nece} by combining  \prettyref{thm:PlantedGeneralSharp_nec} and \prettyref{lmm:max}.
\begin{proof}[Proof of \prettyref{thm:sdp-nece}]
By assumption, $\liminf_{n \to \infty} \pprob{\widehat{Z}_{\SDP}=Z^\ast} >0$. It follows from 
\prettyref{thm:PlantedGeneralSharp_nec} that  with a non-vanishing probability,
\begin{align}
  \min_{i \in C^\ast} \sum_{j \in C^*} A_{ij}  - 	 \max_{i \notin C^\ast} \sum_{j \in C^*} A_{ij}
    \ge  \sup_{0\le a \le K } \left\{   V_{n-K}(a)   -  \frac{a}{K} \max_{i \notin C^\ast} \sum_{j \in C^*} A_{ij}   \right\}.
    \label{eq:gapcondition}
\end{align}
In \prettyref{app:minW} we show that 
\begin{equation}
  \min_{i \in C^\ast} \sum_{j \in C^*} A_{ij}   \leq (K-1) \mu -
	   \sqrt{2 (K-1) \log K} + o_P(\sqrt{K}) . 	\label{eq:minW}
\end{equation}
In view of \prettyref{lmm:extreme}, 
 \begin{equation}
	 \max_{i \notin C^\ast} \sum_{j \in C^*} A_{ij}    \geq
	 \sqrt{2 K \log (n-K)} + o_P(\sqrt{K}).
		\label{eq:maxW}
\end{equation}
It follows from \prettyref{eq:gapcondition} that with a non-vanishing probability, 
\begin{align}
 & (K-1) \mu -
	   \sqrt{2 (K-1) \log K}  - \sqrt{2 K \log (n-K)} + o (\sqrt{K})    \nonumber \\  
  &   \ge  \sup_{0\le a \le K } \left\{   V_{n-K}(a)   - a \sqrt{2 \log (n-K)/K }  \right\}.
    \label{eq:gapconditionx}
\end{align}

\paragraph{Case 1:} $K = \omega(\sqrt{n})$. We show that the necessary condition \prettyref{eq:sdp-nece1} holds.
In view of \prettyref{eq:gapconditionx},  to get a necessary condition as tight as possible, one should  choose $a$ so that $V_{n-K}(a) $ is large and $a$ is small comparing to $K$. To this end, set $a = \sqrt{K} (n-K)^{1/4}$. Since $K=o(n)$ and $K = \omega(\sqrt{n})$ by assumption, we have $a = \omega(\sqrt{n-K})$ and $a = o(K)$. Applying \prettyref{lmm:max}, we conclude that 
\begin{align}
V_{n-K}(a)     \geq & ~ \frac{\sqrt{n-K}}{2} +o_p(\sqrt{n-K}) .
	\label{eq:etalb}
\end{align}
Combining \prettyref{eq:gapcondition}, \prettyref{eq:minW}, \prettyref{eq:maxW}, and \prettyref{eq:etalb}, 
and using $\sqrt{n-K} \geq \sqrt{n} - K/(2\sqrt{n-K}),$
we obtain the desired \prettyref{eq:sdp-nece1}.

\paragraph{Case 2:} $K=O(\sqrt{n})$.
In view of the high-probability lower bounds to $V_{n-K} (a)$ for $a=O(\sqrt{n-K})$ given in \prettyref{eq:max}, 
$V_{n-K}(a)   - a \sqrt{2 \log (n-K)/K}$ is maximized over $[1,K]$ at $a =K$. Hence, we set $a=K$, which satisfies $a=O(\sqrt{n-K})$.
It follows from \prettyref{eq:gapconditionx} that with a non-vanishing probability,
$$
 (K-1) \mu -  \sqrt{2 (K-1) \log K} + o(\sqrt{K}) \ge V_{n-K}(K).
$$
The desired lower bound on $\mu$ follows from the high-probability lower bounds on $V_{n-K}(K)$ given in \prettyref{eq:max} for $a=O(\sqrt{n-K})$. 
\end{proof}

\subsubsection{Bernoulli case}

Recall that $m=n-K$ and by assumption, $L=A$. 
In the Bernoulli case, $M$  is an $m \times m$ symmetric random matrix with zero diagonal and
independent entries such that $M_{ij} = M_{ji} \sim \Bern(q)$ for all $i<j$. 
The following lemma provides a high-probability lower bound on $V_m (a)$ defined in \prettyref{eq:Vm}; its proof is deferred to 
\prettyref{app:proofVmBernoulli}.

\begin{lemma}[Lower bound to $V_m(a)$ in Bernoulli case]
Assume that $a = o(m)$, $q$ is bounded away from $1$,  $m^2 q \to \infty$.
Recall that $\kappa$ is defined in \prettyref{eq:kappa}. With probability tending to one, 
\begin{itemize}
\item If $a -1 \ge  \frac{1}{\kappa} \sqrt{m q/(1-q)}$, then
\[
 V_m(a) \geq (a-1) q+ \frac{\sqrt{m q (1-q) }}{\kappa} .
\]
\item If $0\le a -1 \le \frac{1}{\kappa} \sqrt{m q/(1-q)}$, then $V_m(a) = a-1$.
\end{itemize}
	\label{lmm:maxBernoullie}
\end{lemma}
\begin{remark}
We have the following simple observations on $V_m(a)$:
\begin{itemize}
\item $V_m(1)=0$ and  $V_m(a) \le (a-1) \|A\|_\infty  = a-1 $.
\item
	Dropping the second and the last constraints in \prettyref{eq:Vm}, we have with high probability
	$V_m(a) \leq \lambda_{\max}(A)  \le \kappa \sqrt{mq(1-q)} $. 
\end{itemize}	
\end{remark}

We next prove \prettyref{thm:PlantedSubgraphSharp_sdp_nec} by combining  \prettyref{thm:PlantedGeneralSharp_nec} and \prettyref{lmm:maxBernoullie}.

\begin{proof}[Proof of \prettyref{thm:PlantedSubgraphSharp_sdp_nec}]
We first show that if $Z^*$ is unique with some non-vanishing probability, then $K-1 \ge \sqrt{nq/(1-q)}/\kappa$.
We prove it by contradiction. Suppose that $K-1 < \sqrt{(n-K) q/(1-q)}/\kappa$.
Let $\tA$ denote the $(n-K) \times (n-K)$ submatrix of $A$ supported on $(C^\ast)^c \times (C^\ast)^c$.
Take $a=K$ in  \prettyref{lmm:maxBernoullie}; the last statement of the
lemma implies that $V_{n-K} (K)=K-1$ with high probability.   Furthermore, the proof of the lemma
shows that the $(n-K) \times (n-K)$ matrix $\tZ$ defined by $\tZ_{ii}=1/(n-K)$
and $\tZ_{ij}=(K-1) \tA_{ij}/ \Iprod{\tA}{\allones}$ for $i \neq j$ satisfies
$\Iprod{\tZ}{\tA}=K-1$ and, with high probability, $\tZ \succeq 0$.   Let $Z$ be the $n \times n$ matrix
such that $Z_{(C^\ast)^c (C^\ast)^c}=K  \tZ$ and $Z_{ij}=0$ for all $(i,j) \notin (C^\ast)^c  \times (C^\ast)^c$.
Then one can easily verify that $Z$ is feasible for \prettyref{eq:PDSCVX_SL} with high probability
and $\Iprod{Z}{A} = K(K-1)$.   Since $\Iprod{Z^*}{A}\leq K(K-1)$, it follows that  with high
probability $Z^\ast$ is not the unique optimal solution to \prettyref{eq:PDSCVX_SL}, arriving at
a contradiction.   
The necessity of \prettyref{eq:sdp-necc-BernXX} is then proved.

Next, we prove the necessary condition \prettyref{eq:sdp-necc-Bern}.
Since $\liminf_{n \to \infty} \pprob{\widehat{Z}_{\SDP}=Z^\ast} >0$ by assumption,
\prettyref{thm:PlantedGeneralSharp_nec} implies that with a non-vanishing probability, 
\begin{align}
  \min_{i \in C^\ast}  \sum_{j \in C^*} A_{ij} - 	 \max_{i \notin C^\ast} \sum_{j \in C^*} A_{ij}    \ge  \sup_{1 \le a \le K } \left\{   V_{n-K}(a)   -  \frac{a}{K} \max_{i \notin C^\ast} e(i, C^*)   \right\}. \label{eq:gapcondition_b}
\end{align}

We use the following lower bounds for the binomial distribution tails \cite[Theorem 1]{ZubkovSerov13}:
\begin{align}
   \prob{\Binom(m,p) \le m \tau }  &\ge Q \left( \sqrt{2m d(\tau \| p) } \right),   \quad 1/m \le \tau \le p, \label{eq:ChernoffBinomp} \\
   \prob{\Binom(m,q) \ge m \tau }  & \ge Q \left( \sqrt{2m d(\tau \| q) } \right) , \quad q \le \tau \le 1.   \label{eq:ChernoffBinomq}
\end{align}

Let
$$
\delta= \max \left \{  \frac{2 \log \log K}{\log K}, \frac{\log \log (n-K) }{\log (n-K) } \right\},
$$
and define 
\begin{align*}
\tau'_1 &=(1-\delta) \tau_1 + \delta p \\
\tau'_2 & = (1-\delta) \tau_2 + \delta q.
\end{align*}
Let $K_o = \lceil \frac{K}{\log K}\rceil$ and $\sigma^2 = (K_o-1)  p$.
Define  events 
\begin{align*}
\calE_1 &= \left\{ \min_{i \in C^*} \sum_{j \in C^*} A_{ij}  \le  (K-K_o) \tau_1' + (K_o-1) p + 6 \sigma \right \} \\
\calE_2 &= \left\{ \max_{i\notin C^* }  \sum_{j \in C^*} A_{ij}   \geq K \tau'_2 \right\}.
\end{align*}
By the definition of $\tau'_2$ and the convexity of divergence, we have that
$d(\tau'_2 \|q) \le (1-\delta) d (\tau_2 \| q)$. Thus
\begin{align*}
\prob{ \Binom(K,q) \geq K \tau'_2 }&  \geq Q\left( \sqrt{2K d(\tau'_2 \|q )}  \right)  \\
& \geq Q\left( \sqrt{2K (1-\delta) d (\tau_2 \| q) }  \right)  \\
& = Q\left( \sqrt{2 (1-\delta) \log (n-K) }  \right) \\
& = \Omega \left( \frac{(n-K)^{-(1-\delta)} }{  \sqrt{ \log (n-K) }  } \right) \\
& =  \Omega \left( \frac{ \sqrt{ \log (n-K)}  }{ n-K  } \right), 
\end{align*}
where we used the bound $Q(x) \ge \frac{1}{\sqrt{2\pi}} \frac{t}{t^2+1} \eexp^{-t^2/2}$
and the fact that $\delta \ge \frac{\log \log (n-K) }{\log (n-K) }$.  
Hence,
\begin{align*}
\prob{ \max_{i\notin C^* } \sum_{j \in C^*} A_{ij}  \geq K \tau'_2 }
& \overset{(a)}{=} 1- \prod_{i\notin C^* } \prob{   \sum_{j \in C^*} A_{ij}  < K \tau'_2 } \\
&   \overset{(b)}{=} 1- \left( 1 -\prob{ \Binom(K,q) \geq K \tau'_2 }   \right)^{n-K} \\
& \overset{(c)}{\ge}  1- \exp \left(  - (n-K) \prob{ \Binom(K,q) \geq K \tau'_2 } \right) \\
& \ge  1- \exp \left( - \Omega \left( \sqrt{ \log (n-K) } \right) \right) \to 1,
\end{align*}
where $(a)$ holds due to the independence of $\sum_{j \in C^\ast} A_{ij}$ for different $i \notin C^\ast$; 
$(b)$ holds because for $i \notin C^*$, $\sum_{j \in C^\ast} A_{ij} \sim \Binom(K,q)$;
$(c)$ follows from the fact that $1-x \le \eexp^{-x}$ for $x\ge 0$.
Hence, we get that $\prob{\calE_2} \to 1$.

In \prettyref{app:minA} we show that $\prob{\calE_1} \to 1$, \ie,
\begin{equation}
\prob{ 	\min_{i\in C^* } \sum_{j \in C^*} A_{ij}  \leq (K-K_o) \tau_1' + (K_o-1) p + 6 \sigma  } \to 1.
\label{eq:minA_Bern}
\end{equation}
Let $\calE=\calE_1 \cap \calE_2$.
Then by union bound, $\prob{\calE} \to 1$.  
It follows from \prettyref{eq:gapcondition_b} that with a non-vanishing probability, 
\begin{align}
(K-K_o) \tau_1' + (K_o-1) p + 6 \sigma -  K \tau'_2   \ge  \sup_{1 \le a \le K } \left\{   V_{n-K}(a)   - a \tau'_2 \right\}
\label{eq:gapcondition_bx} 
\end{align}

Applying \prettyref{lmm:maxBernoullie}, we have that with probability converging to $1$,
\begin{align}
  V_{n-K}(a)   \ge
  \begin{cases}
   (a-1) q + \frac{1}{\kappa}  \sqrt{(n-K) q(1-q) }   &  a -1 \ge  \frac{1}{\kappa} \sqrt{ (n-K) q /(1-q)}  \\
    a-1 &  0 \le a -1 \le \frac{1}{\kappa}  \sqrt{ (n-K) q /(1-q)}.
   \end{cases}
  \label{eq:etalowerbound_bernoulli}
\end{align}
Recall that we have shown that  $K-1 \geq \frac{1}{\kappa} \sqrt{ (n-K) q /(1-q)} $ in the first part of the proof.
In view of  $\tau'_2 \ge q$  and \prettyref{eq:etalowerbound_bernoulli}, 
$ V_{n-K}(a)   - a \tau'_2 $ is maximized at 
$$
a=    \frac{1}{\kappa} \sqrt{ (n-K) q /(1-q)}  +1 \in [1, K],
$$ 
which gives $V_{n-K}(a) = a-1$. 
Hence, it follows from \prettyref{eq:gapcondition_bx} that
\begin{align*}
(K-K_o) \tau_1' + (K_o-1) p + 6 \sigma -  K \tau'_2   \ge  a-1  - a \tau'_2,
\end{align*}
which further implies that 
\begin{align*}
(K-K_o) ( \tau'_1 -  \tau'_2 )  & \ge   a-1 - \tau'_2 a  + K_o \tau'_2 - (K_o-1) p - 6 \sigma \\
& = (a-1) (1-\tau'_2) + (K_o-1) ( \tau'_2 - p) - 6 \sigma  \\
&\ge  \frac{1}{\kappa} \sqrt{ \frac{ (n-K) q}{ 1-q} }  (1 -\tau'_2) - 6 \sqrt{ \frac{K p}{\log K} }
  - \frac{K(p-q) }{ \log K} .
\end{align*}
Plugging in the definition of $\tau'_1$ and $\tau'_2$, we derive that
\begin{align*}
(K-K_o)  \left( \tau_1 -\tau_2  \right) &  \ge \frac{1}{\kappa} \sqrt{ \frac{(n-K) q}{1-q} }  (1 -\tau_2') -  6 \sqrt{ \frac{K p}{\log K} } - \frac{K(p-q) }{ \log K}  - \delta (p-q) \\
& \ge \frac{1}{\kappa} \sqrt{ \frac{(n-K) q}{1-q} }  (1 -\tau_2) -  6 \sqrt{ \frac{K p}{\log K} }
-\frac{K (p-q) ( 2\log \log K +1) }{\log K} ,
\end{align*}
where the last inequality follows because $\tau'_2 \le \tau_2$ and $\delta \le \frac{2 \log \log K}{ \log K}$. 
Hence, we arrive at  the desired necessary condition \prettyref{eq:sdp-necc-Bern}.
\end{proof}

\subsubsection{Multiple-community stochastic block model}   \label{sec:pf-sbm}
\begin{proof}[Proof of \prettyref{thm:SBM_sdp_nec}]
Since the MLE is optimal, in proving the theorem,
we can assume without loss of generality that the necessary condition
for consistency of the MLE, $K (p-q)^2 = \Omega( q \log n )$, holds  (see \prettyref{rmk:ChenXu}).
Since $p=\Theta(q)$, it follows that we can assume without loss of generality that 
$K(p-q) =\Omega(\log n)$ and $Kq = \Omega(\log n).$

Suppose \prettyref{eq:sdp-necc-Bern_SBM} fails, namely, there exists $\epsilon>0$ such that
\begin{align}
\frac{(p-q)\sqrt{np}}{rq}   \leq  \frac{1-\epsilon}{\kappa}.
\label{eq:sdp-necc-Bern_SBM-c}
\end{align}
 We construct a matrix $Y$ which, with high probability, constitutes a feasible solution 
 to the SDP program \prettyref{eq:SBM_multiple} with an objective value exceeding that of $Y^\ast$. 
The construction is a variant of that used in proving \prettyref{lmm:maxBernoullie} in \prettyref{app:proofVmBernoulli}.
Let 
\begin{equation}
Y = s A + t(\allones - \identity) + w(\bfd \ones^\top + \ones \bfd^\top - 2D)  + \identity,	
	\label{eq:Yprimal}
\end{equation}
where $\bfd=A\ones$ is the vector of node degrees, $D = \diag{\bfd}$, $s \geq 0$ and $t,w \in \reals$ are to be specified. 
In other words, $Y_{ij} = s A_{ij} + t + w(d_i+d_j)$ for $i\neq j$ and $Y_{ii}=1$. 
Let $z \triangleq \iprod{A}{\allones} = \iprod{\bfd}{\ones}$. 
Note that for any $Y \succeq 0$, the constraint $\Iprod{Y}{\allones}=0$ is equivalent to $Y \ones = 0$.
Since
\[
Y \ones = s \bfd + t(n-1)\ones+ w (n \bfd + z\ones - 2\bfd) + \ones = (s + w(n-2))\bfd + (t(n-1)+wz+1) \ones,
\]
to satisfy $Y \ones = 0$, we let
\[
s + w(n-2) = 0,\quad t(n-1)+wz+1=0
\]
namely,
\begin{equation}
w = - \frac{s}{n-2}, \quad t = \frac{sz}{(n-1)(n-2)} - \frac{1}{n-1}.	
	\label{eq:beta1}
\end{equation}
Since $w \leq 0$, to satisfy the other constraints  in \prettyref{eq:SBM_multiple}, it suffices to ensure
\begin{align}
t + 2 w d_{\max}& \ge - \frac{1}{r-1} \label{eq:SDP_nece_converse_cond1} \\
Y & \succeq 0,  \label{eq:SDP_nece_converse_cond3}
\end{align}
where $d_{\max} = \max_i d_i$ is the maximal degree.

Since $Y \ones=0$, \prettyref{eq:SDP_nece_converse_cond3} is equivalent to $PYP \succeq 0$, where $P=\identity - \frac{1}{n} \allones$ is the
matrix for projection onto the subspace orthogonal to $\ones$. Since 
\[
PYP=P(sA+(1-t)\identity - 2 w D)P,
\] 
in view of the facts that $\Expect[A] \succeq -p \identity$, $D \succeq 0$, and $w \leq 0$, it suffices to verify that 
\begin{equation}
s \|A-\Expect[A]\| \leq 1-t- sp.
	\label{eq:SDP_nece_converse_cond3B}
\end{equation}


Next, we compute the objective value: 
\[
\Iprod{A}{Y} = (s + t) \Iprod{A}{\allones} + 2 w \|\bfd\|_2^2.
\]
By the Chernoff bounds for binomial distributions,  
\begin{align*}
\Iprod{A}{Y^\ast}  & = \frac{n^2} {r} (p-q) + O_P \left(  \sqrt{ n^2 p /r } \right) \\
\Iprod{A}{\allones} & = \frac{n^2} {r}  \left(p + (r-1) q \right ) + O_P ( \sqrt{ n^2 q}  ).
\end{align*} 
Then $\Iprod{A}{Y^\ast} = nK(p-q)(1+o_P(1))$ and $z = \Iprod{A}{\allones} =n^2 q(1+o_P(1))$.
By concentration,\footnote{We
use the following implication of the Chernoff bound:  If $X$ is the sum of independent Bernoulli random variables with mean
$\mu$, then for $\delta \geq  2\eexp -1,$   $\prob{X \geq (1+\delta)\mu} \leq 2^{-\delta \mu},$ and the assumptions
$Kq = \Omega(\log n)$ and $r\to\infty.$} $\|\bfd\|_2^2 =n^3 q^2(1+o_P(1))$ and $d_{\max} = nq(1+o_P(1))$.
To ensure that $\Iprod{A}{Y} > \Iprod{A}{Y^*},$  we set $\Iprod{A}{Y} = (1+\epsilon) \Iprod{A}{Y^\ast},$  or equivalently:
\begin{equation}
(s+t)z + 2 w \|\bfd\|_2^2 = (1+\epsilon) \Iprod{A}{Y^\ast},
	\label{eq:beta2}
\end{equation}
Solving \prettyref{eq:beta1} and \prettyref{eq:beta2} and by the assumption $p = o(1)$ and the fact 
$\frac{1}{n-1}=o(\frac{p-q}{r}),$  we have:
\begin{equation}
s = (1+\epsilon + o_P(1)) \frac{p-q}{rq}, \quad 
t =  (1+\epsilon + o_P(1))  \frac{p-q} {r} , \quad 
w = -(1+\epsilon + o_P(1) ) \frac{p-q}{nrq}.
	\label{eq:s}
\end{equation}
Hence $t + 2 w d_{\max} = - (1+\epsilon+o_P(1) ) \frac{p-q}{r} \geq - \frac{1}{r-1}$, \ie, \prettyref{eq:SDP_nece_converse_cond1}, holds with high probability.


It remains to verify \prettyref{eq:SDP_nece_converse_cond3B}.
Since $np =\Omega(\log n)$, applying \prettyref{lmm:concen_bern} yields
$\|A-\Expect[A]\| \leq \kappa \sqrt{np}$ with high probability. 
In view of the assumption \prettyref{eq:sdp-necc-Bern_SBM-c}, \prettyref{eq:SDP_nece_converse_cond3B} holds with high probability, which completes the proof.
\end{proof}
\section*{Acknowledgement}
This research was supported by the National Science Foundation under
Grant CCF 14-09106, IIS-1447879, NSF OIS 13-39388, and CCF 14-23088, and
Strategic Research
Initiative on Big-Data Analytics of the College of Engineering
at the University of Illinois, and DOD ONR Grant N00014-14-1-0823, and Grant 328025 from the Simons Foundation.
This work was done in part while J. Xu was visiting the Simons Institute for the Theory of Computing.

\bibliographystyle{abbrv}
\bibliography{../graphical_combined}

\appendix

\section{Bounds on spectral norms of random matrices}

For the convenience of the reader, this section collects known bounds on
the spectral norms of random matrices that are used in this paper.

\begin{lemma}[Gordon-Davidson-Szarek]   \label{lmm:GDS}
If $Y$ is an $n\times n$  random matrix such that the random variables
$(Y_{ij}: 1\leq i \leq j \leq n)$ are independent, Gaussian, with mean zero and $\var (Y_{ij}) \leq 1$
then   $\prob{ \norm{Y}\geq 2\sqrt{n} + t } \leq 2\exp(-t^2/4)$ for any $t>0$. 
\end{lemma}

\prettyref{lmm:GDS} is a slight generalization of  \cite[Theorem 2.11]{Davidson01}, which applies to the case of $\var (Y_{ij}) = 1$ and is based on Gordon's inequality on the expected norm $\Expect[\|Y\|]$, proved, in turn, by the Slepian-Gordon comparison lemma. Examining the proof shows that the assumption can be weakened
to $\var (Y_{ij}) \leq 1$.

\begin{lemma}[
{\cite[Theorem 2]{Latala05}}]  
\label{lmm:Latala}
 There is a universal constant $C$ such that
whenever $A$ is a random matrix (not necessarily square) of independent and zero-mean entries:
$$
\expect{\|  A  \| } \leq C\left(     \max_i \sqrt{   \sum_j \eexpect{a_{ij}^2}}    +
  \max_j \sqrt{   \sum_i \eexpect{a_{ij}^2}}   +
  \sqrt[4]{      \sum_{ij} \eexpect{a_{ij}^4}        }     
       \right)
$$
\end{lemma}

\begin{lemma}[Corollary of 
{\cite[Theorem A]{BY88}}]   \label{lmm:BaiYin}
Let $W=(X_{ij}$, $i, j \geq 1)$ be a symmetric infinite matrix such that entries above the diagonal are mean zero
iid, entries on the diagonal are iid, and the diagonal of $W$ is independent of the off-diagonal.
Let $W_n=(X_{ij} :   i , j \in \{ 1, \ldots , n\} )$  for $n\geq 1$.   Let $\sigma^2=\var (X_{12})$.    If 
$\expect{X_{11}^2} < \infty$ and  $\expect{X_{12}^4} < \infty,$ 
then $\| W_n\| / \sqrt{n}  \to 2 \sigma$ a.s.  as $n\to \infty$.
\end{lemma}

The following two lemmas are used in the proof of \prettyref{lmm:concen_bern} below.

\begin{lemma}[
{\cite[Corollary 3.6]{BVH14}}]  \label{lmm:BvH}
 Let $X$ be an $n\times n$ symmetric random matrix with $X_{ij}= \xi_{ij} b_{ij}$,  where
 $\{ \xi_{ij} : i \geq j \}$ are independent symmetric random variables with unit variance,
 and $\{ b_{ij}: i\geq j\}$   are given scalers.   Then for any $\alpha \geq  3$,
 \begin{align*}
 \expect{\| X\| } \leq \eexp^{2/\alpha}   \left\{   2\sigma +  
 14 \alpha  \max_{ij}   \|  \xi_{ij}  b_{ij}  \| _ {2\lceil \alpha \log n \rceil} \sqrt{  \log n },
 \right\}
 \end{align*}
 where $\sigma^2 \triangleq \max_i \sum_j b_{ij}^2$.
 \end{lemma}

\begin{lemma}[
{\cite[Theorem 1.4]{vu07}}]   \label{lmm:Vu}
There are universal constants $C$ and $C'$ such that the following holds.
Let $A$ be a symmetric random matrix such that $\{A_{ij}: 1 \leq i \leq j \leq n\}$ are independent, 
zero-mean, variance at most $\sigma^2$,  and bounded in absolute value by $K$.
If $K$ and $\sigma$ depend on $n$ such that  $\sigma \geq C'n^{-1/2}K\log^2 n$, then
\begin{align}  \label{eq:VuBnd}
\| A\|  \leq 2\sigma \sqrt{n} + C(K\sigma)^{1/2}n^{1/4}\log n,
\end{align}
with probability converging to one as $n\to\infty$.
\end{lemma}
For example, for the case the matrix entries are $\Bern(p)$, the second term in
\eqref{eq:VuBnd} becomes asymptotically negligible compared to the first
if $\sqrt{pn} = \omega( (np)^{1/4}\log n)$, or equivalently,  $np =\omega(\log^4 n)$.

\begin{lemma}\label{lmm:concen_bern}
Let $M$ denote a symmetric $n\times n$
random matrix with zero diagonals and independent
entries such that $M_{ij} =M_{ji} \sim \Bern(p_{ij} )$
for all $i<j$ with $p_{ij} \in [0,1]$.
Assume $p_{ij}(1-p_{ij})\leq r$ for all $i<j$
 and
$n r =\Omega(\log n)$. Then, with high probability,
\begin{align*}
\| M- \expect{M} \| \le \kappa  \sqrt{ n r },
\end{align*}
where 
\begin{align}
\kappa= 
\begin{cases}
O(1)  & nr =\Omega(\log n) \\
4+o(1) & nr = \omega(\log n)  \\
2+o(1)  &  nr =\omega(\log^4 n)
 \end{cases}.
\end{align}
 \end{lemma}
\begin{proof}
It follows from the symmetrization argument and \prettyref{lmm:BvH}  (for this application of the lemma, $b_{ij} \leq \sqrt{r},$
$|\xi_{ij}b_{ij}| \leq 1$, and $\sigma^2 \leq nr$)
that for any $\alpha \ge 3$,
    \begin{align*}
   \expect{\| M -\expect{M} \|}  \le 2  \expect{ \| (M - \expect{M} ) \circ E \|}  \le 2 \eexp^{2/\alpha} ( 2  \sqrt{ n r} +  14 \alpha  \sqrt{ \log n } )
    \end{align*}
where $E$ is an $n \times n$  zero-diagonal, symmetric random matrix whose entries are Rademacher and independent from $M$.
Since $nr=\Omega(\log n)$, we have that  $\expect{\| M -\expect{M} \|} =O(\sqrt{nr})$.
If $nr=\omega(\log n)$, then by letting $\alpha=(nr/\log n)^{1/4} $, we have that $\expect{ \| M- \expect{M} \| } \le (4+o(1)) \sqrt{nr}$.
Talagrand's inequality for Lipschitz convex functions (see \cite{tao.rmt} or \cite[Theorem 7.12]{boucheron2013concentration})
implies that for any constant $c'>0$, there exists a constant $c_0>0$ such that with
probability at least $1-n^{-c'}$, $\| M- \expect{M} \| \le \expect{ \|M - \expect{M} \| } + c_0 \sqrt{\log n} $. 
Hence,  we have proved the lemma for the case of $nr=\Omega(\log n)$ and the case of $nr=\omega(\log n)$. Finally, if $nr=\omega(\log^4 n)$,
then the lemma is a direct consequence of Vu's result, \prettyref{lmm:Vu}, with $K=1$ and $\sigma^2=r$.
 \end{proof}

\section{A concentration inequality for a random matrix of log normal entries}
\label{app:gw}
Let $g(x) = e^{\tau x - \tau^2/2}$ for some $\tau>0$.
Recall that $W$ is an $m \times m$ symmetric, zero-diagonal random matrix with i.i.d.~standard Gaussian entries
up to symmetry.  
Let $g(W)$ denote an $m \times m$ symmetric, zero-diagonal random matrix whose $(i,j)$-th entry 
is $g(W_{ij})$ for $i \neq j$. 
We need the following matrix concentration inequality for $g(W)$.
\begin{lemma}\label{lmm:concentration_primal}
There exists a universal constant $C>0$ such that
\begin{equation}
	\expect{ \| g(W) - \expect{g(W)} \| } \le C \sqrt{m (\eexp^{ 3 \tau^2} -1) }.
	\label{eq:gW1}
\end{equation}
In addition, if $\tau \to 0$ as $m \to \infty$, then the following refined bound holds:
\begin{equation}
\prob{\| g(W) - \expect{g(W)} \| > 2 \sqrt{ m } \tau +  \Delta} \leq O(\sqrt{\tau}) + 2 \eexp^{-m \tau/4},	
	\label{eq:gW2}
\end{equation}
where $\Delta = 2 \sqrt{m} \tau^{3/2} =o(\sqrt{m} \tau )$.
\end{lemma}
\begin{proof}
We first prove \prettyref{eq:gW1}.
Let $U$ be the  upper-triangular part of $g(W) - \expect{g(W) }$. Then $\expect{ \| g(W) - \expect{g(W) } \| } \le 2 \Expect[ \| U\|]$.
Note that $U$ consists of independent zero-mean entries, applying Lata{\l}a's theorem (\prettyref{lmm:Latala}),
we have for some universal constant $c'>0$,
 \begin{align}
 \expect{\| U \|}
 \le & ~c' \pth{ \max_i \sqrt{\sum_j \Expect[U_{ij}^2]} +  \max_j \sqrt{\sum_i \Expect[U_{ij}^2]} +  \sqrt[4]{\sum_{i,j} \Expect[U_{ij}^4]}  } \\
 \le & ~c' \sqrt{m} \pth{ 2 \sqrt{\Expect[U_{12}^2]} +  \Expect[U_{12}^4]^{1/4}  } .
 \label{eq:latala}
 \end{align}
 Note that
 \begin{equation*}
\Expect[U_{12}^2] = \Expect[(e^{\tau W_{12} -\tau^2/2 } - 1)^2 ] = e^{\tau^2} - 1	
\end{equation*}
 Similarly,
 \begin{align*}
 \Expect[U_{12}^4] &= \eexp^{6 \tau ^2}-4 e^{3 \tau ^2} +6 e^{\tau ^2} -3.
 \end{align*}
Combining the last three displays gives that
$\expect{\| U \|}  \le C_0 \sqrt{m} \eexp^{3 \tau^2/2}$ holds for any $\tau > 0$ and some universal constant $C_0>0$.
To complete the proof of \prettyref{eq:gW1}, it remains to show that $\expect{\| U \|}  \le c \sqrt{m (e^{\tau^2} - 1)}$ for all $\tau \in [0,1]$.
 Indeed,
 \begin{align*}
 \Expect[U_{12}^4] &= \eexp^{6 \tau ^2}-4 e^{3 \tau ^2} +6 e^{\tau ^2} -3 \nonumber \\
 & = g(e^{\tau^2}-1) \leq 200 ( e^{\tau^2}-1 )^2 \leq 800 \tau^4,
 \end{align*}
 where $g(s)\triangleq (s+1)^6 - 4 (s+1)^3 + 6 (s+1) - 3 = s^2(3 + 16 s + 15 s^2 + 6 s^3 + s^4) \leq 200 s^2$ for all $ s\in [0,2]$. Applying \prettyref{eq:latala} again yields the desired result. \\

Next we establish the finer estimate \prettyref{eq:gW2} for $\tau \to 0$.
The main idea is to linearize the function $g$. To this end, let $h(x) = g(x) - 1 - \tau x$. Since $\expect{g(W)}=\allones -\identity$,
it follows that
$$
g(W) - \expect{g(W)} = \tau W + h(W).$$
\prettyref{lmm:GDS} yields that
$\prob{ \norm{W}\geq 2\sqrt{m} + t } \leq 2\exp(-t^2/4)$ for any $t>0$.
Hence,
\begin{align}   \label{eq:proberrorW}
\prob{ \| \tau W\| \geq 2 \sqrt{m}\tau  + \sqrt{m} \tau^{3/2}  } \le  2\eexp^{-m \tau/4}.
\end{align}
To bound $\|h(W)\|$, let $B$ be the upper-triangular part of $h(W)$, namely, $B_{ij} = h(W_{ij})$ if $i<j$ and 0 elsewhere. Then $\|h(W)\| \leq 2 \|B\|$. 
Since $B$ consists of independent zero-mean entries, \prettyref{lmm:Latala} yields
 \begin{align*}
 \expect{\| B \|}
 \le & ~c \pth{ \max_i \sqrt{\sum_j \Expect[B_{ij}^2]} +  \max_j \sqrt{\sum_i \Expect[B_{ij}^2]} +  \sqrt[4]{\sum_{i,j} \Expect[B_{ij}^4]} } \\
 \le & ~c \sqrt{m} \pth{ 2 \sqrt{\Expect[B_{12}^2]} +  \Expect[B_{12}^4]^{1/4}  } 
 \end{align*}
 for some universal constant $c$.
 Note that
 $$\Expect[B_{12}^2] = \Expect[(e^{\tau W_{12} - \tau^2/2} - 1 - \tau W_{12})^2] = e^{\tau^2} - 1 - \tau^2 = O(\tau^4)$$
 as $\tau \to 0$. Similarly,
 $$
 \Expect[B_{12}^4] = e^{6 \tau ^2}-4 e^{3 \tau ^2} \left(3 \tau ^2+1\right)+6 e^{\tau ^2} \left(4 \tau ^4+5 \tau ^2+1\right)-3 -4 \tau ^6-21 \tau ^4-18 \tau ^2 = O(\tau^8).
 $$
 Consequently, $\expect{\| B \|}  = O(\sqrt{m} \tau^2)$. Therefore
 \begin{align*}
 \prob{\|h(W)\| \geq \sqrt{m} \tau^{3/2}} \leq \prob{\|B\| \geq \sqrt{m} \tau^{3/2}/2} = O(\sqrt{\tau}).
 \end{align*}
Combining the last displayed equation with \prettyref{eq:proberrorW} and applying the union bound complete the proof for the case $\tau \to 0$.
\end{proof}

\section{Useful facts on binary divergences}

\begin{lemma}\label{lmm:divBound}
For any $0 < q \le p <1$,
\begin{align}
\frac{(p-q)^2}{2 p (1-q ) }  & \le d(p\|q ) \le \frac{(p-q)^2 }{ q(1-q)}.   \label{eq:DivergenceBound} \\
\frac{(p-q)^2}{2 p (1-q ) }  & \le d(q\|p ) \le \frac{(p-q)^2 }{ p(1-p)}.  \label{eq:DivergenceBound2}
\end{align}
\end{lemma}
\begin{proof}
The upper bound follows by applying the inequality $\log x \le x-1$ for $x >0$ and the lower bound is proved using $\frac{\partial^2d(p\|q ) }{\partial p^2}=\frac{1}{p(1-p)}$
and Taylor's expansion.
\end{proof}

\begin{lemma}\label{lmm:divergencecontinuity}
Assume that $0 < q \le p <1$ and $u, v \in [q,p]$. Then for any $0<\eta<1$,
\begin{align}
d( (1-\eta) u + \eta v \| v)  &\ge \left( 1- \frac{2 \eta p (1-q) }{q(1-p) } \right) d(u \| v), \label{eq:divrobust1}   \\
d( (1-\eta) u + \eta v \| u)  & \ge \frac{ \eta^2q (1-p) }{2 p(1-q) } \max \{ d(u \| v), d( v\| u) \} \label{eq:divrobust2}.
\end{align}
\end{lemma}
\begin{proof}
By the mean value theorem,
\begin{align*}
d( (1-\eta) u + \eta v  \| v ) = d( u \|v ) - \eta (u-v) d'(x \| v ),
\end{align*}
for some $x \in (\min\{u, v\}, \max\{u, v\} )$. Notice that $d'(x\| v) =  \log \frac{x(1-v)}{(1-x) v}$ and thus
\begin{align*}
| d'(x \| q) |  \le   \log \frac{ \max\{u, v\} (1- \min\{u, v\} )} {\min\{u, v\}  (1-\max\{u, v\} )  } \le \frac{ |u-v| }{ q (1-p )},
\end{align*}
where the last equality holds due to $\log (1+x) \le x$ and $ x \in (q, p).$
It follows that
\begin{align*}
d( (1-\eta) u + \eta v  \| v )  \ge d(u \| v ) -  \frac{ \eta (u-v)^2 }{  q (1-p ) }  \ge \left( 1- \frac{2 \eta  p (1- q ) }{ q (1-p ) } \right) d( u  \| v ),
\end{align*}
where the last inequality holds due to  the lower bounds in \prettyref{eq:DivergenceBound} and \prettyref{eq:DivergenceBound2}. Thus the first claim follows.
For the second claim,
\begin{align*}
d( (1-\eta) u + \eta v \| u)  \ge \frac{ \eta^2 (u-v)^2 }{ 2 p (1- q ) } \ge \frac{ \eta^2  q  (1-p)  } {2 p (1-q) }   \max \{ d(u \| v), d( v\| u) \},
\end{align*}
where the first inequality holds due to the lower bounds in  \prettyref{eq:DivergenceBound} and \prettyref{eq:DivergenceBound2};
the last inequality holds due to  the upper bounds in \prettyref{eq:DivergenceBound} and \prettyref{eq:DivergenceBound2}.
\end{proof}

\begin{lemma}\label{lmm:taupq}
Assume that $\log \frac{p(1-q)}{q(1-p)}$ is bounded from above.
Suppose  for some $\epsilon > 0$ that  $ K d(p\|q) > (1+ \epsilon)  \log \frac{n}{K}$ for all sufficiently large $n$.
Recall that $\tau^\ast$ is defined in \prettyref{eq:deftau}.
 Then $p-\tau^\ast =\Theta(p-q)$ and $\tau^\ast- q=\Theta(p-q)$.
\end{lemma}
\begin{proof}
By the definition of $\tau^\ast$,
\begin{align*}
p-\tau^\ast &=\frac{d(p \| q) - \frac{1}{K} \log \frac{n}{K} }{ \log \frac{p(1-q)}{q(1-p)} }, \\
\tau^\ast -q & =  \frac{d(q \| p) + \frac{1}{K} \log \frac{n}{K} }{ \log \frac{p(1-q)}{q(1-p)} }.
\end{align*}
Notice that $d(p\|q) + d(q\|p) = (p-q) \log \frac{p(1-q)}{q(1-p)}$. Hence,
\begin{align*}
\frac{p-\tau^\ast}{p-q}  &=\frac{d(p \| q) - \frac{1}{K} \log \frac{n}{K} }{d(p \| q)   + d(q \| p)  }, \\
\frac{\tau^\ast -q}{p-q} & =  \frac{d(q \| p) + \frac{1}{K} \log \frac{n}{K} }{ d(p \| q)   + d(q \| p)    }.
\end{align*}
By the boundedness assumption of  $\log \frac{p(1-q)}{q(1-p)}$ and \prettyref{lmm:divBound},
$d(p \| q) \asymp d(q \| p)$. Since $ K d(p\|q) > (1+\epsilon) \log \frac{n}{K}$ for all sufficiently large $n$,
it follows that $p-\tau^*$ and $\tau^*-q$ are both $\Theta(p-q).$
\end{proof}

\begin{lemma} \label{lmm:tau_1tau_2}
Assume that $\log \frac{p(1-q)}{q(1-p)}$ is bounded. Suppose that $ K d(p\|q) > (1+ \epsilon)  \log \frac{n}{K}$ for all sufficiently large $n$.
\begin{itemize}
\item If $\liminf_{n \to \infty}  \frac{K d(\tau^\ast \|q) }{\log n} \ge 1$, then $\tau_1$ and $\tau_2$ in \prettyref{eq:tau12} are well-defined and take values in the interval $[q,p]$.
\item If $\liminf_{n \to \infty}  \frac{K d(\tau^\ast \|q) }{\log n} >1$, then there exists a fixed constant $\eta>0$ such that 
$\tau_1 \ge (1-\eta) \tau^\ast + \eta p$ and 
$\tau_2 \le (1-\eta) \tau^\ast + \eta q $.
\end{itemize}
\end{lemma}
\begin{proof}
It follows from \prettyref{lmm:taupq} that 
$p-\tau^\ast=\Omega(p-q)$ and $\tau^\ast -q =\Omega(p-q)$. In particular, there
exists a fixed constant $\delta>0$ such that 
$ (1-\delta) q  + \delta p \le \tau^\ast \le (1-\delta ) p + \delta q $. 
By the monotonicity and convexity of divergence, 
$d(\tau^\ast \| q) \le (1-\delta) d(p\|q)$ and $d(\tau^\ast \|p ) \le (1-\delta) d(q\|p)$.
Hence, if $\liminf_{n \to \infty}  \frac{K d(\tau^\ast \|q) }{\log n} \ge 1$, then $K d(p\|q) \ge (1+\delta') \log n$ and $K d(q\|p) \ge (1+\delta') \log K$ for some fixed constant $\delta'>0$.
Thus, in view of the continuity of binary divergence functions, $\tau_1$ and $\tau_2$ are well-defined, 
and moreover $\tau_1 \ge q$ and $\tau_2 \le p$.  

Note that $(1-\eta) \tau^\ast + \eta p \in [q,p]$.
In view of \prettyref{lmm:divergencecontinuity}, 
$$
d\left( (1-\eta) \tau^\ast + \eta q  \| q\right) \ge  \left( 1- \frac{2 \eta p (1-q) }{q(1-p) } \right) d(\tau^\ast  \| q)
$$
If $\liminf_{n \to \infty}  \frac{K d(\tau^\ast \|q) }{\log n} >1$, then there exists a fixed constant $\epsilon'>0$ such that 
for sufficiently large $n$, $K d(\tau^\ast q) \ge (1+\epsilon' ) \log n$. It follows from the last displayed equation that
by choosing $\eta$ sufficiently small, 
$d\left( (1-\eta) \tau^\ast + \eta q  \| q\right) \ge (1+ \delta') \log n$ for some fixed constant $\delta'>0$. Thus by definition,
$\tau_2 \le  (1-\eta) \tau^\ast + \eta q $. Similarly, one can verify that $\tau_1 \ge (1-\eta) \tau^\ast + \eta p$. 
\end{proof}

\section{Proof of \prettyref{cor:lognregime}}\label{app:bern}
We first show that if $\gamma_1>\gamma_2$, then 
\begin{align}
\liminf_{n \to \infty} \frac{K d(\tau^\ast \|q) }{\log n} >1, \label{eq:desiredMLE_nec}
\end{align}
which implies that MLE achieves exact recovery in view of  \prettyref{eq:infor_exact_Bernoulli}.

Recall that $I(x,y)=x-y \log (\eexp x /y)$ for $x, y>0.$
Define $\tau_0 = \frac{a-b}{\log (a/b)}$. Then $I( b, \tau_0)= I(a, \tau_0)$. 
Note that $I(b, \gamma_2)=I(a, \gamma_1)=1/\rho$. Since $I(b,x)$ is strictly increasing 
over $[b, \infty)$ and $I(a,x)$ is strictly decreasing over $(0, a]$, 
it follows that $ \gamma_2<\tau_0<\gamma_1$. Thus $I(b, \tau_0) >1/\rho$. 
In the regime \prettyref{eq:lognregime}, we have 
$ \tau^\ast = \frac{\log^2 n}{n} \left(\tau_0 + o(1)\right)$.
 Taylor's expansion yields that 
$$
d(\tau \|q ) = q- \tau \log \frac{\eexp q}{\tau } + O( (\tau-q)^2) = I(q, \tau)+ O( (\tau-q)^2).
$$
Therefore,
$$
d(\tau^\ast \|q) = \frac{\log^2 n}{n} \left(  I (b, \tau_0) +  o(1) \right),
$$
which implies the desired \prettyref{eq:desiredMLE_nec}. 

Secondly, suppose that  MLE achieves exact recovery. We aim to show that $\gamma_1 \ge \gamma_2$.
Suppose not. Then $\gamma_1< \gamma_2$. By the similar argument as above,
it follows that $\gamma_1<\tau_0<\gamma_2$. 
Thus $I(b, \tau_0)<1/\rho$. As a consequence,
$$
\frac{K d(\tau^\ast \|q) }{\log n} \le 1-\epsilon,
$$
for some positive constant $\epsilon>0$, which contradicts the 
fact that $\liminf_{n \to \infty} \frac{K d(\tau^\ast \|q) }{\log n}  \ge 1$, the necessary 
condition \prettyref{eq:infor_exact_Bernoulli_necess} for MLE to achieve exact recovery. 

Finally, we prove the claims for SDP.  By definition, $\tau_1= \log^2 n (\gamma_1+o(1))/n$
and $\tau_2=\log^2 n (\gamma_2+o(1))/n$. 
Therefore, if $\rho (\gamma_1-\gamma_2) >4 \sqrt{b}$, then the sufficient condition for SDP \prettyref{eq:sdp-suff-Bern} holds;
if the necessary condition for SDP \prettyref{eq:sdp-necc-Bern} holds, then  $\rho (\gamma_1-\gamma_2) \ge \sqrt{b}/4$.

\section{Proof of \prettyref{lmm:max}}
\label{app:proofVmGaussian}
\begin{proof}
To prove the desired lower bound of $V_m(a)$, 
we construct an explicit  feasible solution $Z$ to \prettyref{eq:Vm}.
For a given $\tau \in \reals$, let 
\begin{align*}
g(x) = e^{\tau x - \tau^2/2} , \quad \alpha = \frac{2}{m(m-1)} \sum_{i<j} g(W_{ij}) >0.
\end{align*}
Define an $m \times m$ matrix $Z$ by $Z_{ii}= \frac{1}{m}$ and $Z_{ij} = \frac{ a-1 }{\alpha m (m-1) } g(W_{ij})$ for $i \neq j$. 
By definition, $Z \ge 0, \Tr(Z)=1$, and $\Iprod{Z}{\allones} = a$.

We pause to give some intuition on the construction of $Z$. Note that $g$ is in fact the likelihood ratio between two shifted Gaussians: $g(x)=\frac{\diff \calN(\tau,1)}{\diff \calN(0,1) }(x)$ and thus $\expect{g(W_{ij})} =1$ and  $\expect{W_{ij} g(W_{ij}) } = \tau$. 
Therefore we expect that $\alpha$ is concentrated near $1$, and similarly
\begin{align*}
\Iprod{Z}{W} & \approx \frac{2 (a-1)}{\alpha m(m-1)} \sum_{i<j} \expect{W_{ij} g(W_{ij} )}   = \frac{(a-1) \tau}{\alpha}  \approx (a-1) \tau.
\end{align*}
Thus to get a lower bound to $V_m(a)$ as tight as possible, we would like to maximize 
$\tau$ so that $Z \succeq 0$ with high probability. 

Recall that as defined in \prettyref{app:gw}, $g(W)$ is an $m \times m$ zero-diagonal matrix with the $(i,j)$-th entry
given by $g(W_{ij})$. Then  $\expect{g(W ) } = \allones- \identity$ and
\begin{align*}
Z & = \frac{1}{m} \identity +  \frac{ a-1 }{\alpha m (m-1) }  g(W)  \\
 &= \frac{1}{m} \identity +  \frac{ a-1 }{\alpha m (m-1) }  (\allones -\identity)
+   \frac{ a-1 }{\alpha m (m-1) } \left(  g(W)  - \expect{g(W ) } \right).
\end{align*}
Hence, since $\allones \succeq 0$ and $a\geq 1$,  to show $Z \succeq 0$, it suffices to verify that
\begin{equation}
\| g(W)  - \expect{g(W)} \|  \leq \frac{\alpha(m-1)}{a-1} - 1.
	\label{eq:Zpsd}
\end{equation}
Therefore, we aim to choose $\tau$ as large as possible to satisfy \prettyref{eq:Zpsd}.
Recall that  \prettyref{lmm:concentration_primal}  provides different upper bounds on $ \| g(W)  - \expect{g(W)} \|$
depending on the value of $\tau$. Accordingly, the best choice of $\tau$ depends on
the particular regimes of $a$.

\paragraph{Case 1:} $\omega(\sqrt{m}) \le a \le o(m)$.
	Let
\begin{equation}
\tau = \frac{\sqrt{m}}{2 (a-1)} - \frac{m^{3/4}}{2\sqrt{2} (a-1)^{3/2}}  -  \frac{1}{\sqrt{m}}. 
	\label{eq:mu}
\end{equation}
The last two terms  in \prettyref{eq:mu} are lower order terms comparing to the first term; thus $\tau=(1+o(1)) \frac{\sqrt{m}}{2 (a-1)}$. It follows that for sufficiently large $m$, $\tau > 0$, $\tau =o(1)$, and $\tau = \omega(1/\sqrt{m})$. \\

We next show that $Z$ is feasible for \prettyref{eq:Vm} with high probability; it suffices to verify \prettyref{eq:Zpsd}. 
For $i<j$, $\expect{g(W_{ij} )}=1$ and $\var\left(g(W_{ij})\right) = \eexp^{\tau^2}-1  =O(\tau^2)$.
It follows from Chebyshev's inequality that
\begin{align}
\prob{ \Bigg| \sum_{i<j} \left( g(W_{ij}) -\expect{ g(W_{ij}) } \right) \Bigg|  \ge \frac{1}{2} \sqrt{m} (m-1) \tau }
\leq \frac{2(e^{\tau^2}-1)}{(m-1)\tau^2}   \to 0.
\label{eq:alpha}
\end{align}
Thus, with probability tending to one,
\begin{equation}
| \alpha-1| \leq \tau/ \sqrt{m} . 
	\label{eq:alpha1}
\end{equation}
As a consequence, we have that
\begin{align}
 \frac{\alpha(m-1)}{a-1} - 1 \ge \frac{m-1}{a-1} - \frac{\sqrt{m} \tau }{a-1} -1 \ge \frac{m-1}{a-1}- \frac{m}{2(a-1)^2} -1. \label{eq:psd1}
\end{align}

Since $\tau \to 0$ and $\tau m \to \infty$, applying \prettyref{lmm:concentration_primal} yields that with probability tending to one,
\begin{equation}
\| g(W) - \expect{g(W)} \| \le 2 \sqrt{m} \tau + 	
2 \sqrt{m} \tau^{3/2}.
	\label{eq:gWgW}
\end{equation}
Plugging in the definition of $\tau$ given in  \prettyref{eq:mu}, \prettyref{eq:gWgW} implies that with probability converging to one,
\begin{align}
\| g(W) - \expect{g(W)} \|  \le \frac{m}{a-1}  - 2 . \label{eq:psd2}
\end{align}

Since  by assumption $a=\omega(\sqrt{m})$ and $a=o(m)$, combining \prettyref{eq:psd1} and \prettyref{eq:psd2}
yields that with probability tending to one, \prettyref{eq:Zpsd} and hence $Z \succeq 0$ hold. \\

Finally, we compute the value of the objective function $\Iprod{Z}{W}$.
For $i<j$, $\expect{W_{ij} g(W_{ij})} = \tau$ and $\expect{W^2_{ij}g^2(W_{ij})} = \eexp^{\tau^2} ( 1+ 4 \tau^2 )=O(1)$.
Thus, it follows from Chebyshev's inequality (and $a=\omega(m)$) that:
\begin{align}
\prob{ \Bigg| \sum_{i<j} \left( W_{ij} g(W_{ij}) - \tau  \right) \Bigg| \ge  \frac{m(m-1)}{2}\left[\frac{1}{\sqrt{m}} -
\frac{1}{a-1}\right] } = O \left(\frac{1}{m} \right).
\label{eq:ZW}
\end{align}
In view of \prettyref{eq:mu} and \prettyref{eq:alpha1}, with probability tending to one,
\begin{align*}
\Iprod{Z}{W} & = \frac{2(a-1)}{\alpha m(m-1)}  \sum_{i<j}  W_{ij} g(W_{ij} )  \\
&\geq  \frac{(a-1) \tau}{\alpha} - \frac{a}{\alpha \sqrt{m}}  + \frac{1}{\alpha}   \\
&\geq  \left\{ (a-1) \tau - \frac{a}{\sqrt{m}}  +  1   \right\}\left(1-\frac{\tau}{\sqrt{m}}\right)  \\
&\geq  \left\{  \frac{\sqrt{m}}{2}  -  \frac{ m^{3/4}}{\sqrt{8(a-1)} }      - \frac{2a}{\sqrt{m}}  +  1   \right\}\left(1-\frac{1}{2(a-1)}\right)  \\
& \geq \frac{\sqrt{m}}{2} -  \frac{ m^{3/4}}{\sqrt{8(1-a)} }  -  \frac{2a}{\sqrt{m}} ,
\end{align*}
proving the first part of \prettyref{eq:max}.

\paragraph{Case 2:} $a = o(\sqrt{m})$.   The desired lower bound given in the third part of \prettyref{eq:max} is trivially true for $a=1$
so we suppose $a \geq 2$.  The proof is almost identical to the first case except that we set
\begin{equation}
	\tau =\sqrt{ \frac{1}{3} \pth{\log  \frac{m}{ a^2}  - \log \log \frac{m}{ a^2} }}.
	\label{eq:tau2}
\end{equation}
First, we verify that  \prettyref{eq:Zpsd}  holds with high probability. 
By the choice of $\tau$, $e^{\tau^2} = o(m^{1/3})$. Thus,  \prettyref{eq:alpha} and hence \prettyref{eq:alpha1} continue to hold. 
It follows from \prettyref{eq:alpha1} that $\alpha = 1 + O_P(\frac{\log(m/a^2)}{\sqrt{m}})$.
 Applying \prettyref{lmm:concentration_primal} and Markov's inequality, with probability at least $1 - (\log \frac{m}{a^2})^{-1/4}$,
\begin{align*}
	\| g(W) - \expect{g(W)} \|  \le C \Big(\log \frac{m}{a^2}\Big)^{1/4} \sqrt{m (\eexp^{ 3 \tau^2} -1) }.
\end{align*}
Plugging in the definition of $\tau$ given in \prettyref{eq:tau2}, it further implies that 
with high probability, 
$$
\| g(W) - \expect{g(W)} \|  \le C \frac{m}{a} \Big(\log \frac{m}{a^2}\Big)^{-1/4}.
$$
Therefore \prettyref{eq:Zpsd}  holds with high probability. \\

Then we compute the value of the objective function $\Iprod{Z}{W}$.
Entirely analogously to \prettyref{eq:ZW}, we have
\begin{align*}
\prob{ \Bigg| \sum_{i<j} \left( W_{ij} g(W_{ij}) - \tau  \right) \Bigg| \ge  \frac{\tau m(m-1)}{2 m^{1/3}} }  \leq \frac{2 m^{2/3}  e^{\tau^2}(1+4\tau^2)}{m(m-1)\tau^2}  \to 0.
\end{align*}
Therefore with probability tending to one,
\[
\Iprod{Z}{W} = \frac{2(a-1)}{\alpha m(m-1)}  \sum_{i<j}  W_{ij} g(W_{ij} )  \geq  (a-1) \tau \pth{1 -O\Big(\frac{\log(m/a^2)}{\sqrt{m}}\Big)} \left(1 - m^{-1/3} \right).
\]
By the choice of $\tau$ given in \prettyref{eq:tau2}, we have that 
$$
 \tau \ge \sqrt{ \frac{1}{3} \log  \frac{m}{ a^2} }  \left( 1- O\left(  \frac{ \log \log(m/a^2) }{  \log (m/a^2)}\right) \right). 
$$
Combining the last two displayed equations yield that 
with high probability, 
$$
\Iprod{Z}{W} \geq (a-1) \sqrt{ \frac{1}{3} \log  \frac{m}{ a^2} } - O\left(   \frac{a \log \log(m/a^2) }{ \sqrt{ \log (m/a^2)} }\right), 
$$
proving the desired lower bound to $V_m(a)$ given in the third part of \prettyref{eq:max}. 

\paragraph{Case 3:} $a = \Theta(\sqrt{m})$. Let $\tau$ be a constant to be chosen later.
The proof is similar to the previous two cases; the key difference is that the distributions of entries of $g(W)$
are independent of $m$, and thus we can the invoke \prettyref{lmm:BaiYin},  a corollary of the Bai-Yin theorem,
instead of \prettyref{lmm:concentration_primal}, to obtain
$$
\|g(W)-\Expect[g(W)]\| = 2 \sqrt{m (e^{\tau^2}-1)}(1+o_P(1)).
$$
In view of \prettyref{eq:Zpsd}, as long as $\tau$ is chosen to be a constant so that 
$$
0< \tau < \liminf_{m\diverge}\sqrt{\log\left(1+ \frac{m}{4a^2} \right)} ,
$$ 
we have $Z \succeq 0$ with high probability.

Finally, we compute the value of the objective function $\Iprod{Z}{W}$.
Entirely analogously to \prettyref{eq:ZW}, we have
\begin{align*}
\prob{ \Bigg| \sum_{i<j} \left( W_{ij} g(W_{ij}) - \tau  \right) \Bigg| \ge  \frac{ m(m-1)}{2 a } }  =O(1/m)  \to 0.
\end{align*}
It follows that with probability converging to $1$, 
$$
\Iprod{Z}{W} \geq \frac{(a-1) \tau -1}{\alpha}  \geq \left( (a-1) \tau -1 \right) \left( 1 - O(m^{-1/2} ) \right) = a \tau +O(1) ,
$$
which yields the desired lower bound to $V_m(a)$. 
\end{proof}

\section{Proof of \prettyref{lmm:maxBernoullie} }
\label{app:proofVmBernoulli}
The proof follows the same fashion as that in the Gaussian case. 
In particular, to prove the desired lower bound to $V_m(a)$, 
we construct an explicit  feasible solution $Z$ to \prettyref{eq:Vm};
however, the particular construction is different. Recall that in the 
Bernoulli case, 
$M$ is assumed to be an $m \times m$ symmetric random matrix
with zero diagonal and independent entries such that $M_{ij} = M_{ij}
\sim \Bern (q)$ for all $i<j$. 

Let
$$
R= \frac{\iprod{M}{\allones}}{  m(m-1) } 
$$ 
and assume that $R \in (0, 1)$ for the time being. 
For a given $\gamma \in (0, 1] $,  define 
\[
\alpha = \frac{\gamma-R}{R(1-R)}\frac{(a-1)}{ m(m-1)}, \quad \beta = \frac{1-\gamma}{1-R}\frac{(a-1)}{m (m-1)} \ge 0. 
\]
Define an $m \times m$ matrix $Z$ by $Z_{ii} = 1/m$ and $Z_{ij} = \alpha M_{ij} + \beta$ for $i\neq j$.
By definition, 
$\Iprod{Z}{\identity}=1$, $\alpha+\beta = \frac{\gamma (a-1)}{Rm(m-1)} \ge 0$, and thus $Z \ge 0$. 
Moreover, 
$$
\Iprod{Z}{\allones}  = \alpha \Iprod{M}{\allones} + \beta \Iprod{\allones-\identity}{ \allones} +
\frac{1}{m} \Iprod{\identity}{\allones}= a 
$$
and
\begin{align}
\iprod{Z}{M}  = (\alpha+\beta) \iprod{M}{\allones}  = (a-1) \gamma.  \label{eq:ZAvalue}
\end{align}
Thus to get a lower bound to $V_m(a)$ as tight as possible, we would like to choose 
$\gamma$ as large as possible to satisfy  $Z \succeq 0$ with high probability. 

Note that
\[
Z = \alpha M + \beta(\allones - \identity) + (1/m) \identity = (\beta+ \alpha q)(\allones - \identity) + (1/m) \identity +  \alpha (M-\Expect[M]).
\]
Thus, to show $Z \succeq 0$,  it suffices to verify that
 \begin{align*}
 \frac{1}{m}- \beta -\alpha q -  \| M-\expect{M} \| \ge 0.
\end{align*}
By  \prettyref{lmm:concen_bern}, with high probability,
 $\| M- \Expect[M] \| \le \kappa\sqrt{ m q(1-q)}$, where $\kappa $ is 
 a universal positive constant defined in \prettyref{eq:kappa}. 
 Hence, to show $Z \succeq 0$, it further suffices to verify that
 \begin{align}
 \frac{1}{m}- \beta -\alpha q -  \kappa \alpha \sqrt{ m q (1-q)} \ge 0. \label{eq:psdcondition_b}
\end{align}
As a result,  we would like to choose 
$\gamma \in (0,1]$ as large as possible to satisfy \prettyref{eq:psdcondition_b}.  
We pause to give some intuitions on the choice of $\gamma$. 
By concentration inequalities, $R \approx q$ with high probability. 
Since $a=o(m)$, $\beta=o(1/m)$. Furthermore, $q \ll \sqrt{mq(1-q)}$. Hence, to 
satisfy \prettyref{eq:psdcondition_b}, roughly it suffices that 
$$
\alpha \le \frac{1}{\kappa \sqrt{mq(1-q)} m},
$$
which further implies  that 
$$
\gamma \le q +  \frac{ \sqrt{ mq(1-q) } }{\kappa (a-1) }.
$$
This suggests that we should take $\gamma$
to be the minimum of $q +  \frac{ \sqrt{ mq(1-q) } }{\kappa (a-1) }$ and $1$.

Before specifying the precise choice of $\gamma$, we first show that $R$ is close to $q$ with high probability. 
Let $c_m = \log (m \sqrt{q} ) $ which converges to infinity
under the assumption that $m^2 q \to \infty$. Thus, by the Chernoff bound for the binomial distribution,
with probability converging to $1$, $|R-q| \le c_m \sqrt{q} /m$. Without loss of generality, we can and do assume that
 $|R-q| \le c_m \sqrt{q} /m$ in the remainder of the proof.
Since $q$ is bounded away from $1$ and $m^2 q \to \infty$,  
$R$ is  also bounded away from $1$ and $R>0$. 
This verifies that $\alpha, \beta$ and hence $Z$ are well-defined.

Let
\begin{align}
\gamma = 
\begin{cases}
  q + \left( 1-  \epsilon \right)  \frac{ \sqrt{ mq(1-q) } }{\kappa (a-1) } & a-1 \ge \frac{1-\epsilon}{\kappa} \sqrt{ \frac{mq}{1-q}} \\
  1 &  0 \le a-1 \le \frac{1-\epsilon}{\kappa} \sqrt{ \frac{mq}{1-q}}, 
  \end{cases}
\end{align}
where $\epsilon =2 / \log\left( m \min \{ \sqrt{q}, 1/a \} \right)$.  Equivalently,
\begin{align}
\gamma = \min \left \{ q + \left( 1-  \epsilon \right)  \frac{ \sqrt{ mq(1-q) } }{\kappa (a-1) } , 1 \right \}.  
\label{eq:defgamma2}
\end{align}
The assumptions, $m^2 q \to \infty$ and $a=o(m)$, imply that $\epsilon=o(1)$ and hence $\gamma \in [q, 1]$.

Next, we compute the value of $\Iprod{Z}{M}$. In view of \prettyref{eq:ZAvalue},
it suffices to evaluate $(a-1) \gamma$.  By the choice of $\gamma$, 
\begin{align}
(a-1)\gamma  = 
\begin{cases}
  (a-1)q + (1-\epsilon)   \frac{ \sqrt{ mq(1-q) } }{\kappa}  & a-1 \ge \frac{1-\epsilon}{\kappa} \sqrt{ \frac{mq}{1-q}} \\
  a-1 & 0 \le a-1 \le \frac{1-\epsilon}{\kappa} \sqrt{ \frac{mq}{1-q} }
 \end{cases}.  \label{eq:a_gamma}
\end{align}
Since  $\epsilon=o(1)$,  absorbing the factor $1-\epsilon$ in the last 
displayed equation into the definition of $\kappa$ given in \prettyref{eq:kappa} 
yields the  desired lower  bound to $V_m(a)$.

To finish the proof, we are left to  verify \prettyref{eq:psdcondition_b}. 
Since $\beta +\alpha R= \frac{a-1}{m(m-1)}$, it follows that
\begin{align}
\frac{1}{m}- \beta -\alpha q = \frac{1}{m}- \beta -\alpha R - \alpha(q-R) = \frac{m-a}{m(m-1)}  -   
O\left(  \frac { (a-1) \gamma c_m}{  m^3 \sqrt{q} }  \right),
\label{eq:psd1_bern}
\end{align}
 where we used the fact that $|R-q| \le c_m \sqrt{q} /m$  and $\alpha \le a \gamma /m^2R $ in the last equality.

Let $\alpha_0  = \frac{\gamma-q}{q(1-q)}\frac{(a-1)}{ m(m-1)}$. Next, we bound $|\alpha-\alpha_0|$ from the above. 
In view of $|R-q| \le c_m \sqrt{q} /m$ and $\gamma \ge q$, 
\begin{align*}
\bigg| \frac{\gamma-R}{R(1-R)} - \frac{\gamma-q}{q(1-q)} \bigg| &\le \bigg|  \frac{\gamma-R}{R(1-R)}  -  \frac{\gamma-q}{R(1-R)}   \bigg| + \bigg|   \frac{\gamma-q}{R(1-R)} -  \frac{\gamma-q}{q(1-q)} \bigg| \\
 & \le \frac{|R-q| } { R (1-R) }  + (\gamma-q) \frac{ | R-q| | R+q -1 | }{R(1-R) q(1-q)} \\
 & = O \left( \frac{c_m }{ m \sqrt{q} }  \right) +  O\left(  \frac{ c_m \gamma} { m q^{3/2} }  \right) 
 =O\left(  \frac{ c_m \gamma } { m q^{3/2} }  \right)
 \end{align*}
Consequently,  
\begin{align}
\alpha-\alpha_0 = O \left(  \frac{ (a-1)  \gamma c_m } { m^3 q^{3/2}}  \right). \label{eq:psd2_bern}
\end{align}
Combining  \prettyref{eq:psd1_bern} and \prettyref{eq:psd2_bern} yields that 
\begin{align}
& \frac{1}{m}- \beta -\alpha q -   \kappa \alpha   \sqrt{ mq (1-q)}  \nonumber \\
& = \frac{1}{m}- \beta -\alpha q -   \kappa \alpha_0   \sqrt{ m q (1-q)} -  (\alpha-\alpha_0)  \kappa \sqrt{mq (1-q)} \nonumber \\
& =  \frac{1}{m(m-1)} \left(  m- a  - \frac{ (a-1) (\gamma-q) } {q (1-q)} \kappa \sqrt{mq (1-q) } -  O \left(  \frac{ (a-1) \gamma c_m } {  \sqrt{m} q } \right)     \right ).  \label{eq:Zposcheck}
\end{align}
Thus, to verify \prettyref{eq:psdcondition_b}, it reduces to show the right hand side of the last displayed equation is negative.  
In view of  \prettyref{eq:defgamma2}, 
$$
\frac{ (a-1) (\gamma-q) } {q (1-q) } \kappa \sqrt{mq (1-q) } \le (1-\epsilon )m . 
$$
and 
\begin{align*}
\frac{ (a-1) c_m  \gamma }{  \sqrt{m} q }
 \le \frac{ (a-1)  c_m} {  \sqrt{m} } +    \frac{  c_m }{  \kappa \sqrt{q} }  
=o\left( \frac{m}{ \log( m\sqrt{q} ) } \right),
\end{align*}
where  the last equality because $c_m = \log (m \sqrt{q} )$ and the assumption that $a=o(m)$.
Combining the last two displayed equations and plugging in the definition of $\epsilon$ yield that 
\begin{align*}
& m- a  - \frac{(a-1) (\gamma-q) } {q (1-q)} \kappa \sqrt{mq (1-q) } -  O \left(  \frac{ (a-1) \gamma c_m } {  \sqrt{m} q } \right)   \\
& \ge  \frac{2m}{ \log \left( m \min \{ \sqrt{q}, 1/a \} \right) } -a -o\left( \frac{m}{ \log( m\sqrt{q} ) } \right) \ge 0. 
\end{align*}

Hence, it follows from \prettyref{eq:Zposcheck} that  \prettyref{eq:psdcondition_b}  holds. 
Consequently, $Z \succeq 0$ holds with high probability. 
This completes the proof of the lemma.

\section{Proof of \prettyref{eq:minW}}
	\label{app:minW}
Note that for each $i\in C$, $X_i \triangleq \sum_{j\in C}W_{ij}$ is distributed according to $\calN(0,K-1)$ but not independently. Below we use the Chung-Erd\"os inequality \cite{chung-erdos}:
\begin{equation}
	\prob{\bigcup_{i=1}^K  A_i }  \geq \frac{\pth{\sum_{i=1}^K \prob{A_i}  }^2}{\sum_{i=1}^K \prob{A_i} + \sum_{i\neq j} \prob{A_iA_j}}.
	\label{eq:chung-erdos}
\end{equation}
For any $i \neq j$, $\prob{X_i \leq -s, X_j \leq -s} = \Expect[Q^2((s+Z)/\sqrt{K-2})] \triangleq \Expect[g^2(Z)]$ where $Z \sim \calN(0,1)$, and $\prob{X_i \leq -s} = \Expect[g(Z)] = Q(s/\sqrt{K-1})$. Therefore
\begin{align*}
	& ~ 	\prob{X_i \leq -s, X_j \leq -s} - \prob{X_i \leq -s}\prob{X_j \leq -s} \nonumber\\
= & ~ 	\var(g(Z)) = \expect{\pth{Q\pth{\frac{s+Z}{\sqrt{K-2}}}  - Q\pth{\frac{s}{\sqrt{K-1}}}  }^2}	\\
\leq & ~ 	Q(s/4) + \varphi\pth{\frac{3s/4}{\sqrt{K-2}}}^2 \expect{\pth{\frac{s+Z}{\sqrt{K-2}}  - \frac{s}{\sqrt{K-1}}}^2}	\\
\leq & ~ 	\exp(-s^2/8) + \exp\pth{-\frac{9 s^2/16}{K-2}} \qth{\frac{1}{K-2} + \pth{\frac{1}{\sqrt{K-1}} - \frac{1}{\sqrt{K-2}}  }^2 s^2}.
\end{align*}
Let $s =  \sqrt{K-1} (\sqrt{2 \log K}  - \log \log K / \sqrt{2\log K})$. Then $\prob{X_i \leq -s} = \Theta(\sqrt{\log K}/K)$ and
\[
\prob{X_i \leq -s, X_j \leq -s} - \prob{X_i \leq -s} \prob{X_j \leq -s}   = O(K^{-17/8}).
\]
 Applying \prettyref{eq:chung-erdos}, we conclude that $\prob{\min_{i\in C}  X_i \leq -s }  \geq 1 - O(1/\sqrt{\log K})$.

\section{Proof of \prettyref{eq:minA_Bern}}
\label{app:minA}
We show $\prob{{\cal E}_1} \to 1$.  In this section, by a slight abuse of notation,
let $e(i, S)=\sum_{j\in S} A_{ij}$. 
The proof is complicated by the fact
the random variables $e(i,C^*)$ for $i\in C^*$ are not independent.
The trick is to fix $C^*$ and a small set $T \subset C^*$  with $|T|=K_o$. 
Then for $i\in T$, $e(i,C^*) = e(i,C^*\backslash T) + e(i,T),$  and we can
make use of the fact that the random variables $(e(i,C^*\backslash T) : i\in T)$
are independent,  $(e(i,C^*\backslash T): i\in T)$ is independent of
$(e(i,T): i\in T),$  and, with high probability, at least
half of the random variables in $e(i,T)$ are not unusually large.
(The same trick is used for proving Theorem 6 in \cite{HajekWuXu_one_info_lim15}.)

Suppose for convenience of notation that
$C^*$ consists of the first $K$ indices, and $T$ consists
of the first $K_o$ indices:  $C^*=[K]$ and $T=[K_o]$.
Let $T' =\{i\in T  : e(i,T) \leq  (K_o-1) p + 6\sigma  \},$
Since\footnote{In case $T'=\emptyset$ we use the usual convention that
the minimum of an empty set of numbers is $+\infty$.}
$$
\min_{i\in C^*} e(i,C^*)  \leq  \min_{i \in T'}  e(i,C^*)   \leq   \min_{i\in T'}  e(i,C^*\backslash T) + (K_o-1) p + 6\sigma,$$
it follows that
$$
\prob{E_1}\geq \prob{ \min_{j \in T'}  e(j ,C^*\backslash T)   \leq (K-K_o) \tau'_1 }.
$$
We show next that  $\prob{|T'|  \geq \frac{K_o}{2} }\to 1$ as $n \to \infty$.
For $i \in T,$   $e(i,T)=X_i+Y_i$ where   $X_i=e(i,\{1, \ldots , i-1\})$ and $Y_i=e(i, \{i+1, \ldots , K_o\})$.
The $X$'s are mutually independent, and the $Y$'s are also mutually independent, and  $X_i$ has the
$\Binom(i-1,p)$ distribution and $Y_i$ has the  $\Binom(K_o-i,p)$ distribution.   Then
$\expect{X_i}=(i-1)p$ and  $\var(X_i) \leq \sigma^2$.    Thus, by the Chebyshev inequality,
$ \prob{X_i \geq  (i-1)p + 3\sigma }\leq \frac{1}{9} $
for all $i\in T$.   Therefore, $|\{i : X_i \leq  (i-1)p + 3\sigma  \}|$ is stochastically at least as
large as a $\Binom\left(K_o, \frac{8}{9}\right)$
random variable,   so that,  $\prob{ |\{i : X_i \leq  (i-1)p + 3\sigma    \}|  \geq \frac{3K_o}{4} }\to 1$ as $K_o\to \infty$
(which happens as $n\to \infty)$.
Similarly,  $\prob{ |\{i : Y_i \leq  (K_o-i)p +3 \sigma  \}|  \geq \frac{3K_o}{4} }\to 1$.
If at least 3/4 of the $X$'s are small and at least 3/4 of the $Y$'s are small, it follows that at least
1/2 of the $e(i,T)$'s for $i \in T$  are small.  Therefore,  $\prob{|T'|  \geq \frac{K_o}{2} }\to 1$ as claimed.

The set $T'$ is independent of  $(e(i,C^*\backslash T): i\in T)$  and those variables each
have the $\Binom(K-K_o,p)$ distribution.  Using the tail lower bound \prettyref{eq:ChernoffBinomp},
we have
\begin{align*}
\prob{E_1} & \geq   1- \expect{ \prod_{j \in T'} \prob{e(j, C^\ast\backslash T ) \geq K\tau^* - K_op - 6\sigma  }  \bigg|  |T'|\geq \frac{K_o}{2} } - \prob{|T'| < \frac{K_o}{2}}  \\
& \ge 1- \exp \left(  -  Q \left( \sqrt{2 (K-K_o)  d(\tau'_1  \| p )}  \right) K_o/2  \right) - o(1).
\end{align*}
By definition of $\tau'_1$ and the convexity of divergence,
$  d(\tau'_1 \| p) \le (1- \delta) d( \tau_1 \|p )  $,
it follows that
\begin{align*}
Q \left( \sqrt{2( K- K_o)  d(\tau'_1   \| p )}  \right) K_o/2  & \ge
Q \left( \sqrt{2( K- K_o)(1-\delta )   d( \tau_1  \| p )}  \right) K_o/2  \\
& \ge Q \left( \sqrt{2 (1-\delta  )  \log K}  \right) K_0/2 \ge  \frac{\sqrt{ \log K} }{2},
\end{align*}
and $\prob{E_1} \to 1$.

\end{document}